\documentclass[sigconf,authorversion,nonacm]{acmart}

\usepackage{pifont}
\usepackage{subcaption}
\usepackage{soul}
\usepackage{multirow}
\usepackage{bbm}

\newcommand{\xmark}{\ding{55}}
\newtheorem*{lemma*}{Lemma}
\newtheorem*{theorem*}{Theorem}

\begin{document}

\title[U-Former ODE]{U-Former ODE: Fast Probabilistic Forecasting \\
of Irregular Time Series}

\author{Ilya Kuleshov}
\authornote{Corresponding author: i.kuleshov@applied-ai.ru}
\affiliation{
    \institution{Applied AI Institute}
    \city{Moscow}
    \country{Russia}
}

\author{Alexander Marusov}
\affiliation{
    \institution{Applied AI Institute}
    \city{Moscow}
    \country{Russia}
}

\author{Alexey Zaytsev}
\affiliation{
    \institution{Applied AI Institute}
    \city{Moscow}
    \country{Russia}
}

\renewcommand{\shortauthors}{Kuleshov et al.}

\begin{abstract}
   Probabilistic forecasting of irregularly sampled time series is crucial in domains such as healthcare and finance, yet it remains a formidable challenge. Existing Neural Controlled Differential Equation (Neural CDE) approaches, while effective at modelling continuous dynamics, suffer from slow, inherently sequential computation, which restricts scalability and limits access to global context. We introduce UFO (U-Former ODE), a novel architecture that seamlessly integrates the parallelizable, multiscale feature extraction of U-Nets, the powerful global modelling of Transformers, and the continuous-time dynamics of Neural CDEs. By constructing a fully causal, parallelizable model, UFO achieves a global receptive field while retaining strong sensitivity to local temporal dynamics. Extensive experiments on five standard benchmarks---covering both regularly and irregularly sampled time series---demonstrate that UFO consistently outperforms ten state-of-the-art neural baselines in predictive accuracy. Moreover, UFO delivers up to 15$\times$ faster inference compared to conventional Neural CDEs, with consistently strong performance on long and highly multivariate sequences\footnote{The code and datasets can be found in our anonymized GitHub: \url{https://anonymous.4open.science/r/ufo_kdd2026-64BB/README.md}.}.
\end{abstract}

\maketitle



\section{Introduction}
The task of Probabilistic Multivariate Time Series Forecasting (PMTSF) is ubiquitous in industry and research alike, serving as the critical backbone for decision-making under uncertainty across countless domains. 
From optimizing smart energy 
grids~\cite{zhou2021informer} and forecasting financial market risks~\cite{lai2018modeling} to anticipating patient health trajectories in healthcare~\cite{PhysioNet-challenge-2019-1.0.0}, this discipline moves beyond simple point estimates to model the full joint probability distribution of future events. 

The majority of time series analysis methods focus on regular data, i.e. such that all observations lie on a grid with no gaps~\cite{chung2014empirical,vaswani2017attention}.
However, such is rarely the case: physical sensors are prone to errors, shutdowns, breakdowns, etc.
To add, a subtle point which we would like to highlight in our work is: there are many different types of irregular data.
Most recent works for irregular time series focus on the scenario, where the missing points are uniformly spread throughout the sequence~\cite{kidger2020neural,kuleshov2024denots}.
On the other hand, in the real world, missing data usually comes in blocks. 
More often than not, a sensor malfunctions for extended periods of time until it is fixed, instead of skipping time steps randomly and uniformly.
Consequently, the missing points come in groups.
Such a scenario presents even more of a challenge to non-specialized methods, since the data is significantly harder to impute via simple algorithms such as forward filling~\cite{khayati2020mind}.
The above highlights the need for specialized, irregular-capable methods.

Neural Controlled Differential Equations (CDEs) are a very prominent branch of methods, designed specifically for irregular time series~\cite{kidger2020neural}.
The idea is to integrate a differential equation with a Neural Network as its dynamics function, which continuously takes the current data interpolation as input.
Formally, a most general Neural CDE is described as follows:
\begin{equation} \label{eq:ncde_int}
\frac{d\mathbf{z}}{dt} = \mathbf{z}(0) + \int_0^T \mathbf{f}_\theta (\mathbf{z}(t), \hat{\mathbf{x}}(t)) dt.
\end{equation}
Here, $\mathbf{z}(t): [0, T] \rightarrow \mathbb{R}^d$ is the hidden trajectory of the CDE, $\mathbf{f}_\theta(\mathbf{z}, \hat{\mathbf{x}}): \mathbb{R}^d \times \mathbb{R}^d_x \rightarrow \mathbb{R}^d$ is the Neural Network, which defines the dynamics, and $\hat{\mathbf{x}}): [0, T] \rightarrow \mathbb{R}^{d_x}$ is the interpolation of the input time series, where $d$ is the hidden dimension and $d_x$ is the input dimension.
Given the initial hidden state~$\mathbf{z}(0)$ and a hidden state interpolation function~$\hat{\mathbf{x}}(t)$, a dedicated solver evaluates the hidden trajectory~$\mathbf{z}(t)$ at all the necessary timestamps by integrating~\eqref{eq:ncde_int}.
This way, Neural CDEs effectively mimic the dynamics of the observed process, which makes them a prime choice for the temporal domain.
Unfortunately, CDE integration is often prohibitively expensive.
Due to their recurrent nature, it is very difficult to parallelize them over time.
Moreover, since they store a single hidden state during integration, these models may experience information loss~\cite{kuleshov2024denots}, something RNNs are well-known for~\cite{hochreiter1997long}.
Finally, the dependence of traditional Neural CDEs on smooth interpolation routines makes them poorly adapted for time series forecasting~\cite{morrill2021online}, where the sequence is not known beforehand.

At the same time, the modern Deep Learning (DL) research landscape is lately becoming increasingly focused on Transformers~\cite{vaswani2017attention}, mostly due to their wide success in Natural Language Processing (NLP).
The core idea behind Transformers, the Attention mechanism, has found its way into almost all domains, including Computer Vision~\cite{dosovitskiy2020image}, graphs~\cite{yun2019graph}, recommendation systems~\cite{zhao2024recommender}, and many more.

The task of time series forecasting is no exception, with new highly influential transformer-based methods introduced every year~\cite{li2019enhancing,wu2021autoformer,zhou2021informer,nie2022time}.
However, recent work shows, that Transformers may not be well-suited for time series forecasting due to the permutation-invariant nature of Attention~\cite{zeng2023transformers}.
Originating from the field of NLP, Transformers are designed to work with highly informative tokens, with small correlations between neighbouring values.
Time series usually have high neighbour correlations, and a single observation carries very little meaning, so Attention may struggle with local patterns.
Prior work introduces patching~\cite{nie2022time} and various patching structures, similar to the fully-convolutional U-Net architecture~\cite{ronneberger2015u}, to help Transformers overcome this obstacle~\cite{madhusudhanan2021yformer,you2024kernel}.
On the other hand, U-Nets, like most other fully convolutional networks, often suffer from limited receptive field~\cite{luo2016understanding}, so the U-Net-Transformer hybrid is highly beneficial for both sides.
However, the U-Net Transformer branch of methods is currently under-researched, and is not yet properly adapted to irregular time series forecasting.

\begin{figure}[!t]
    \centering
    \includegraphics[width=\linewidth]{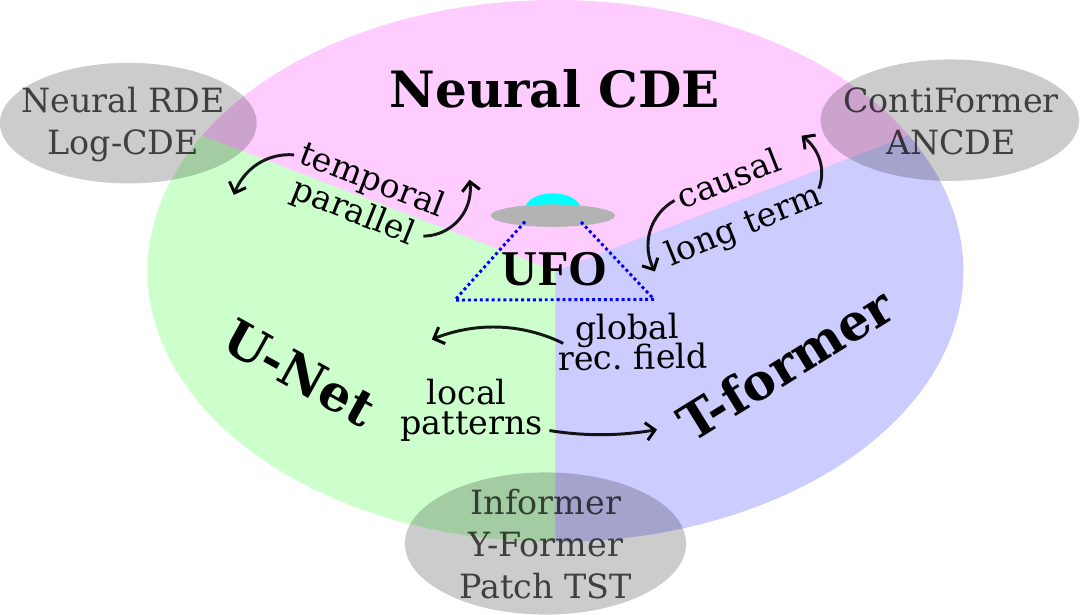}
    \caption{Our UFO method combines three leading time series architectures, each bringing its own strengths and compensating for the others' weaknesses (as illustrated by the arrows). Gray ellipses show the existing works, closest to the intersection, however, none match perfectly, see Section~\ref{sec:review} for details.}
    \Description{This image illustrates what each of the three different architectures brings to UFO. It takes its efficiency and long-term memory from Transformers, while escaping the permutation-invariance of attention through ODE methods.
    It has the temporal, data-mimicking properties of Neural ODEs, can handle irregular data, but has the execution time of a non-ODE method, since it can be parallelized over time via a U-Net-like hierarchical structure.
    It can process features at multiple scales like a U-Net, but does not suffer from a limited receptive field due to the Transformer counterpart.}
    \label{fig:ufo_unity}
\end{figure}

We propose a solution to the problems listed above.
Our novel U-Former ODE (UFO) combines three highly influential architectures, bringing out their strong points and compensating for their drawbacks, as illustrated by Figure~\ref{fig:ufo_unity}.
It takes its efficiency and long-term memory from Transformers, while escaping the permutation-invariance of attention through ODE methods.
It has the temporal, data-mimicking properties of Neural ODEs, can handle irregular data, but has the execution time of a non-ODE method, since it can be parallelized over time via a U-Net-like hierarchical structure.
It can process features at multiple scales like a U-Net, but does not suffer from a limited receptive field due to the Transformer counterpart.

The most important and novel contributions of our paper are as follows:
\begin{itemize}
    \item \textbf{Time-Parallel Continuous Model.}
    UFO is the first time-parallel Neural CDE method for time series forecasting.
    It achieves a 15-fold reduction in inference time in comparison to classic Neural CDEs, outperforming other modern differential alternatives (Section~\ref{sec:performance_evaluation}).
    \item \textbf{Global-Attention Neural CDE.}
    Transformers allow us to solve the information bottleneck present in Neural CDEs and Recurrent Neural Networks. 
    Instead of forcing the network to store all past information in a single hidden state, Attention gathers knowledge from the entire sequence.
    The receptive field regularity of UFO is confirmed in a separate sensitivity analysis study (Section~\ref{sec:sensitivity_analysis}).
    \item \textbf{Efficient Irregular Transformers.} 
    The Neural CDE-based patching algorithm acts as a temporal encoding for the Transformer architecture, levelling out the sequence irregularity and embedding it into the corresponding vectors.
    We prove that the observation grid becomes more regular with patching (Section~\ref{sec:patching_regularization}) and demonstrate that irregularity encoding works in a separate experiment (Section~\ref{sec:irregularity_encoding}).
    \item \textbf{SOTA Performance.} UFO achieves state-of-the-art metrics on 5 popular probabilistic forecasting benchmarks, outperforming 10 baselines, on both regular and irregular data (Section~\ref{sec:qual_res}).
\end{itemize}

\section{Problem Statement}
We study probabilistic time series forecasting. Given $T \in \mathbb{N}$ past observations
\[
\mathbf{x} = (\mathbf{x}_1, \ldots, \mathbf{x}_T), \quad
\mathbf{x}_i = (x_{i,1}, \dots, x_{i,d})^\top \in \mathbb{R}^d,
\]
the goal is to predict $L \in \mathbb{N}$ future values
\[
\mathbf{y} = (\mathbf{y}_1, \ldots, \mathbf{y}_L), \quad
\mathbf{y}_i = (y_{i,1}, \dots, y_{i,d})^\top \in \mathbb{R}^d.
\]
As the task is probabilistic, we expect the models to output $\mathbf{F} : \mathbb{R}^d \to [0,1]^d$, the predictive Cumulative Density Function, CDF.

Alongside observations, we also pass the timestamps $t_1, \dots, t_{T+L}$ to the models for both context and forecast horizon, to handle data irregularity. 
We further augment inputs with cyclic time-of-day features $\tau_1, \dots, \tau_{T+L} \in \mathbb{R}^c$, encoding fractional progress through natural cycles (e.g., day, week, month, year). 
These fractions are mapped to sine–cosine pairs to ensure continuity across cycle boundaries, following standard practice in modern forecasting libraries \cite{gluonts_jmlr}.
Concatenating these features with $\mathbf{x}$, we obtain the inputs used to produce CDF forecasts $\mathbf{F}$ that we aim to make close to the true values $\mathbf{y}$.

\section{Method}
This section describes the proposed U-Former ODE method (UFO).
To improve the reading experience, we first outline the general structure of UFO in Section~\ref{sec:ufo_structure}, and then describe each component of the proposed method in detail in Section~\ref{sec:ufo_details}.

\subsection{Method Structure}
\label{sec:ufo_structure}

\begin{figure}
    \centering
    \includegraphics[width=1.0\linewidth]{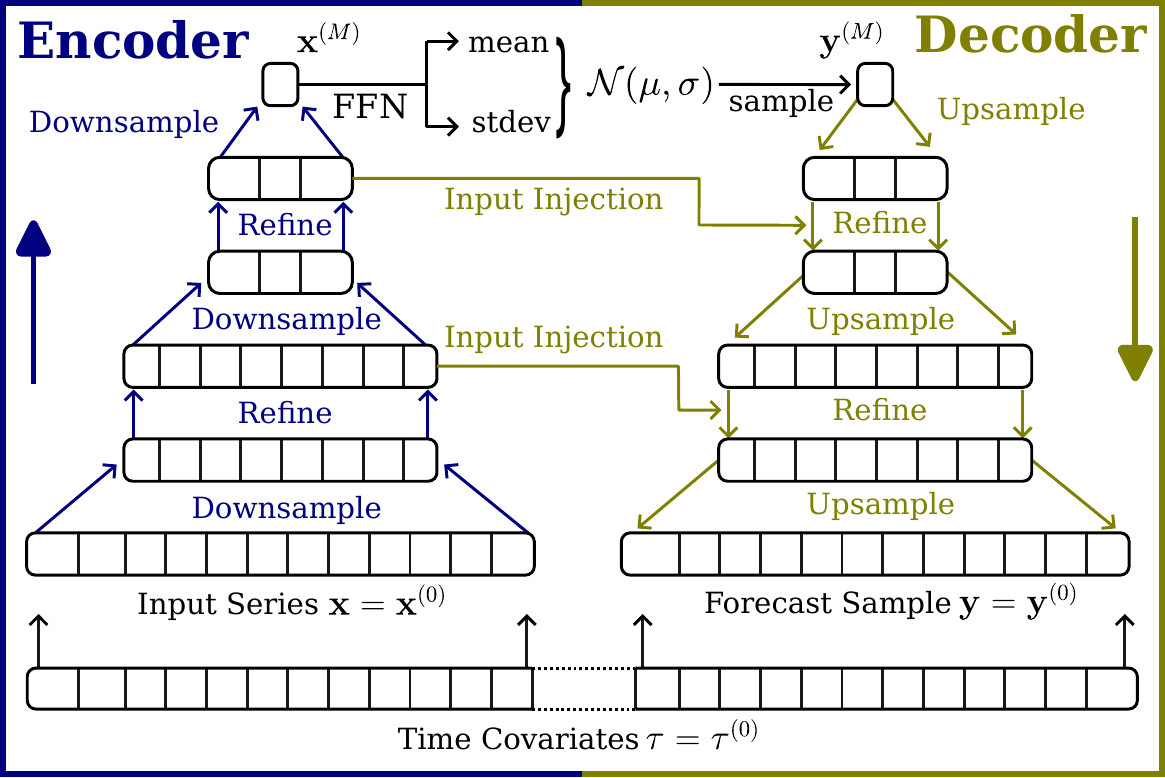}
    \caption{The proposed UFO hierarchical architecture.
    }
    \Description{This image complements the hierarchical architecture description from Section~\ref{sec:ufo_structure}.
    It shows that UFO is an Encoder-Decoder architecture, with Encoder consisting of Downsamplers and Refiners, Decoder -- of Upsamplers and Refiners, with the two parts are connected by skip-connection to Decoder Refiners and by a sampling connection at the top layer.}
    \label{fig:ufo_hierarchy}
\end{figure}

Let $\mathbf{x}^{(0)} := \mathbf{x} \in \mathbb{R}^{T_0 \times d}$ denote the input time series, where $T_0 := T$ is the input length and $d$ the model dimension (inputs are projected to $d$ via a linear layer).
The model follows an encoder--decoder architecture organized into $M+1$ hierarchical levels indexed by $m = 0, \ldots, M$.
At level $m$, the encoder representation is denoted by $\mathbf{x}^{(m)} \in \mathbb{R}^{T_m \times d}$ and the decoder representation by $\mathbf{y}^{(m)} \in \mathbb{R}^{T_m \times d}$, with $T_{m+1} < T_m$.
Level $m=0$ corresponds to the original temporal resolution: $\mathbf{x}^{(0)}$ is the input sequence and $\mathbf{y}^{(0)}$ the output forecast.
The resulting architecture is schematically presented in Figure~\ref{fig:ufo_hierarchy}.

\paragraph{Layer types}
We introduce the neural network layer types used in the UFO architecture: downsampling, upsampling, and refining layers.
The specific resampling architectures are described in Section~\ref{sec:ncde_resampling}, while the refining procedure is detailed in Section~\ref{sec:transformer_refining}.
Downsampling layers reduce sequence resolution in the encoder by partitioning the input into equal-length patches with a fixed stride and aggregating each patch.
Because patches are processed independently, downsampling is fully parallelisable over the time dimension.
Upsampling layers perform the inverse operation in the decoder, expanding embeddings into patches to increase sequence resolution, and are likewise parallelisable.
Refining layers enrich embeddings by mixing information across the sequence to maximise the receptive field.
In the decoder, they also integrate intermediate encoder embeddings via U-Net-style skip connections.
Each level of the architecture consists of a resampling layer (downsampling in the encoder, upsampling in the decoder) followed by a refining layer.

\paragraph{Encoder}
The encoder maps the input sequence to a hierarchy of increasingly abstract representations by alternating refining and downsampling layers.
For $m = 0, \ldots, M-1$, the encoder transition from level $m$ to $m+1$ is defined as
\begin{equation}
\mathbf{x}^{(m+1)}
=
\mathrm{DS}^{(m)}\!\left(
\mathrm{R}_{\mathrm{enc}}^{(m)}(\mathbf{x}^{(m)})
\right),
\label{eq:encoder_transition}
\end{equation}
where $\mathrm{R}_{\mathrm{enc}}^{(m)} : \mathbb{R}^{T_m \times d} \rightarrow \mathbb{R}^{T_m \times d}$ is a refiner, and $\mathrm{DS}^{(m)} : \mathbb{R}^{T_m \times d} \rightarrow \mathbb{R}^{T_{m+1} \times d}$ is a downsampling operator.
Refining is omitted at the bottom level since the input sequence may be highly irregular.
At higher levels the sequence irregularity drops, as we prove in Section~\ref{sec:patching_regularization}.
As $m$ increases, encoder representations capture progressively higher-level temporal structure.

At the top level, the encoder output $\mathbf{x}^{(M)}$ is mapped to the parameters of a diagonal Gaussian latent distribution:
\begin{equation}
\mu = f_\mu(\mathbf{x}^{(M)}), \qquad
\sigma = f_\sigma(\mathbf{x}^{(M)}),
\end{equation}
where $f_\mu$ and $f_\sigma$ are learnable projections.
A latent variable is then sampled as
\begin{equation}
\mathbf{y}^{(M)} \sim \mathcal{N}(\mu, \sigma).
\end{equation}

\paragraph{Decoder}
The decoder generates a forecast by progressively reconstructing representations at finer temporal resolutions.
For $m = M-1, \ldots, 0$, the decoder transition is given by
\begin{equation}
\mathbf{y}^{(m)}
=
\mathrm{R}_{\mathrm{dec}}^{(m)}\!\left(
\mathrm{US}^{(m)}(\mathbf{y}^{(m+1)}),\,
\mathbf{x}^{(m)}
\right),
\label{eq:decoder_transition}
\end{equation}
where $\mathrm{US}^{(m)} : \mathbb{R}^{T_{m+1} \times d} \rightarrow \mathbb{R}^{T_m \times d}$ is an upsampling operator, and $\mathrm{R}_{\mathrm{dec}}^{(m)}$ is a decoder refiner block.
As mentioned above, each decoder refiner also injects the corresponding encoder representation $\mathbf{x}^{(m)}$, conditioning the decoder on varying-resolution features from the input sequence.
At the bottom Decoder level, refining is also omitted (similar to the Encoder). 

Although the top-level latent variable $\mathbf{y}^{(M)}$ is modelled using a diagonal Gaussian distribution, this choice does not restrict the expressiveness of the predictive distribution. The Gaussian serves only as a latent prior at the coarsest temporal scale. After sampling, $\mathbf{y}^{(M)}$ is transformed by a sequence of expressive non-linear upsampling and refining layers that inject multi-scale encoder representations via U-Net skip connections. Consequently, the final output distribution is shaped primarily by the decoder network and can be highly non-Gaussian and multi-modal, despite the Gaussian assumption at the top level.

\subsection{Architecture Details}
\label{sec:ufo_details}
This section provides a detailed description of the specific architectures of the Resampling and Refining components of UFO.

\subsubsection{Neural CDE Resampling}
\label{sec:ncde_resampling}
One of our main contributions is the choice of the family of downsampling and upsampling architectures.
To better align with the nature of time-series data, we utilise Neural Controlled Differential Equations for both.
Scaling time in these models allows us to tune the resolution at which the resampling layer interprets the sequence~\cite{kuleshov2024denots}, ensuring that higher levels only attend to global features.

\begin{figure}
    \centering
    \begin{subfigure}{0.23\textwidth}
    \centering
    \includegraphics[width=\textwidth]{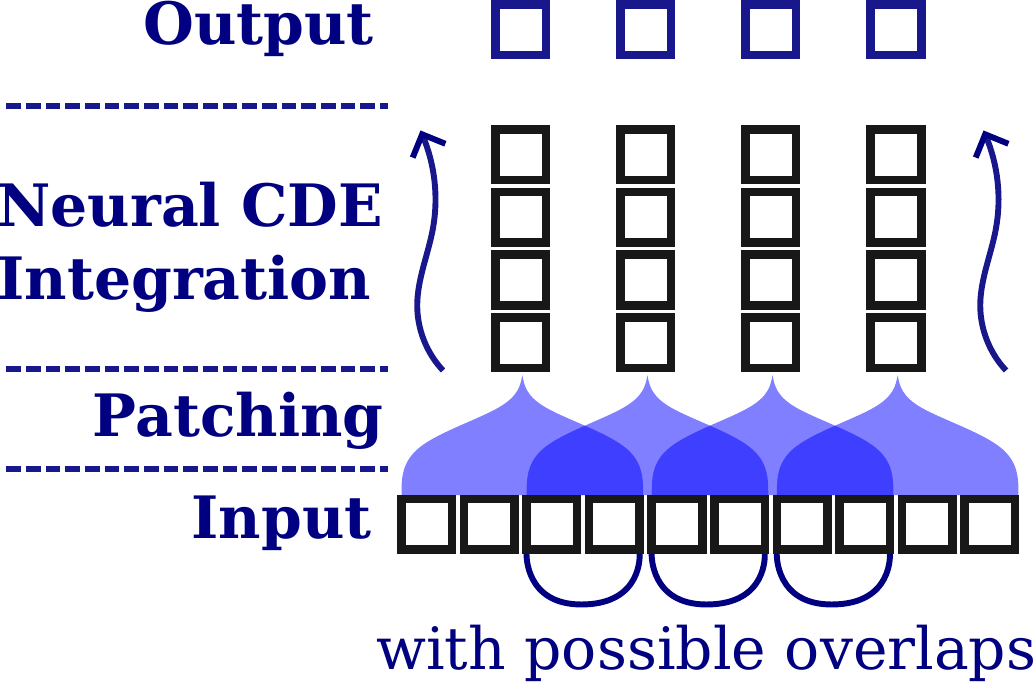}
    \end{subfigure}
    \hfill
    \begin{subfigure}{0.23\textwidth}
    \centering
    \includegraphics[width=\textwidth]{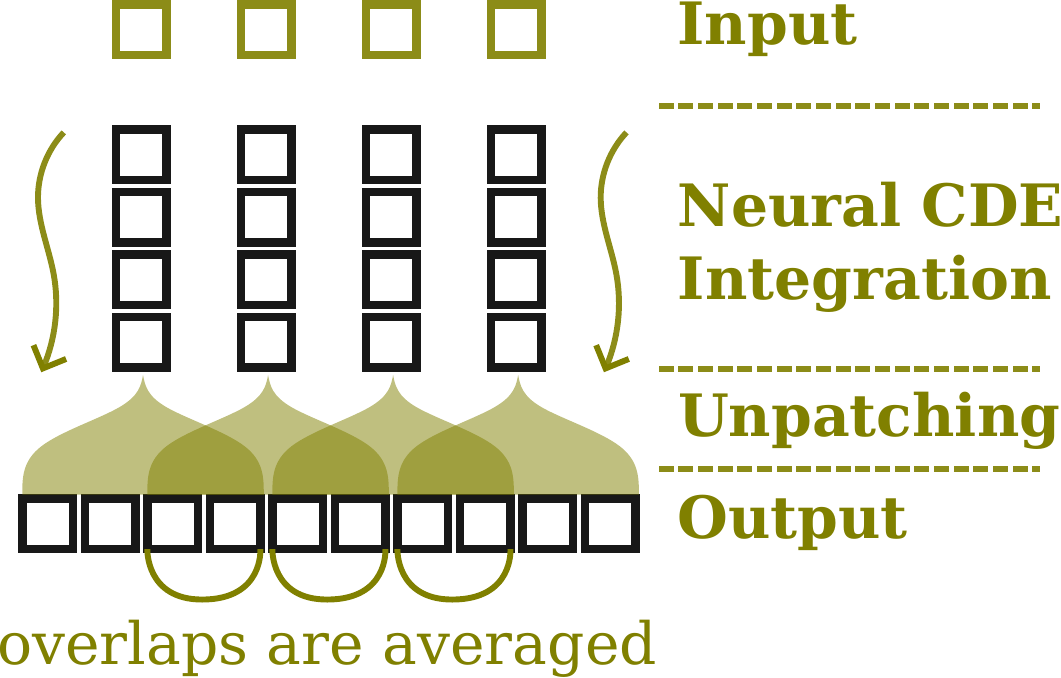}
    \end{subfigure}
    \caption{
    The Neural CDE Downsampling (left) and Upsampling (right) procedures.
    Half-transparent domes represent the (un)patching procedure, the wavy arrows represent Neural CDE integration~\eqref{eq:ncde_int}, which corresponds to~\eqref{eq:ncde_downsampling} for downsampling, and to~\eqref{eq:ncde_upsampling} for upsampling.
    }
    \Description{
    This figure shows the NCDE Down/Upsampling routines, described in Section~\ref{sec:ufo_details}.
    Neural CDE Downsampling: first the sequence is patched, then a Neural CDE is integrated along the patch, and the final states of the hidden trajectories, the outputs, are the output.
    Neural CDE Upsampling, similar to downsampling, but reversed: the input acts as the starting state for NCDE, Neural CDE integration goes along the unfolded timestamps and saves all intermediate states, and these resulting states are then aligned with the expected patches and averaged at overlaps.
    }
    \label{fig:ncde_updownsampling}
\end{figure}

During downsampling, we integrate over each patch, and the trajectory value at the final timestamp of that patch is the corresponding embedding.
The Neural CDE downsampling procedure is illustrated by Figure~\ref{fig:ncde_updownsampling}, left.
The initial state is determined via a feed-forward network, applied to the first observation, as recommended by~\cite{kidger2020neural}.
For simplicity, we write out both resampling procedures for a single patch.
Let's assume that $[t_{s}, t_{e}] \subset \mathbb{R}$ denotes the interval in time of the specific patch we wish to aggregate on level $m$, $\hat{\mathbf{x}}^{(m - 1)}$ is the interpolation of the prior level output, and $\hat{\tau}(t)^{(m)}$ is the interpolation of the time covariates, corresponding to the current level\footnote{Time covariates are the same between levels, but levels have different time scales.}.
Then, the formal Cauchy problem is given by:
\begin{equation} \label{eq:ncde_downsampling}
\begin{aligned}
    \mathbf{z}_\textrm{ds}^{(m)} (t_s) &= \text{NN}_0^{(m)}(\hat{\mathbf{x}}^{(m - 1)}(t_s)); \\
    \frac{d}{dt} \mathbf{z}_\textrm{ds}^{(m)} (t) &= \text{VF}_\textrm{ds}^{(m)}(\hat{\mathbf{\tau}}^{(m)}(t), \hat{\mathbf{x}}^{(m - 1)}(t), \mathbf{z}_\textrm{ds}^{(m)} (t)); \\
        \mathbf{x}^{(m)}\left(\frac{t_e}{w}\right) &:= \mathbf{z}_\textrm{ds}^{(m)}(t_e).
\end{aligned}
\end{equation}
Here, $\mathbf{z}_\textrm{ds}^{(m)}$ is the hidden trajectory, $\text{NN}_0$ is a Neural Network, used to produce the initial hidden state from the first observation, and $\text{VF}_\textrm{ds}^{(m)}$ represents some Neural Network-based vector field\footnote{We use the "ds" and "us" subscripts to indicate downsampling and upsampling correspondingly, and the $(m)$ superscript to show the level number.}.
Note that the resulting representation $\mathbf{x}^{(m)}$ is now known at the final time of the patch $t_e$, divided by patch length $w$.
The division by length is done to keep mean patch duration constant between levels, which is important since it directly influences Neural CDE performance~\cite{kuleshov2024denots}.

For upsampling, we reverse this procedure: using the initial hidden state from the higher level as the starting value, we interpolate the time-based features, and use them as input to a "decoder" Neural CDE.
This is illustrated in Figure~\ref{fig:ncde_updownsampling}, right (note the similarity to downsampling).
Again, let's assume that $[t_{s}, t_{e}] \subset \mathbb{R}$ denotes the target interval in time on level $m$, where the resulting upsampled patch would be, and $\hat{\mathbf{\tau}}^{(m)}$ is the interpolation of the corresponding time covariates.
The upsampling routine is given by:
\begin{equation} \label{eq:ncde_upsampling}
\begin{aligned}
    \mathbf{z}_\textrm{us}^{(m)} (t_s) &= \mathbf{y}^{(m+1)}(t_s); \\
    \frac{d}{dt}\mathbf{z}_\textrm{us}^{(m)} (t) &= \text{VF}_\textrm{us}^{(m)}([\hat{\mathbf{\tau}}^{(m)}(t), \mathbf{z}_\textrm{us}^{(m)} (t)]); \\
    \mathbf{y}^{(m)}(t \cdot w) &= \mathbf{z}_\textrm{us}^{(m)} (t),\; t \in [t_s, t_e].
\end{aligned}
\end{equation}
Similar to~\eqref{eq:ncde_downsampling}, $\mathbf{z}_\textrm{us}^{(m)}$ is the hidden trajectory and $\text{VF}_\textrm{us}^{(m)}$ represents some Neural Network-based vector field.
The function $\mathbf{y}^{(m)}$ is then calculated on a grid, corresponding to the current level and the current patch, with the time upscaled by $w$ (to match the Encoder timescale), and the resulting sequence is used for subsequent refining and upsampling layers.

\paragraph{Interpolation}
To integrate~\eqref{eq:ncde_downsampling} and~\eqref{eq:ncde_upsampling}, we need to define interpolation routines for~$\hat{\mathbf{x}}^{(m)}$ and~$\hat{\mathbf{\tau}}^{(m)}$.
Following the most recent work on efficient Neural CDEs, we replace the traditional cubic spline interpolation routine~\cite{kidger2020neural} with kernel smoothing~\cite{serov2026efficient}.
However, we go one step further.
Standard kernel regression relies on sufficient local data density.
In regions where the design is sparse (i.e., observations are far apart), kernel estimators can exhibit high variance and poor finite-sample behavior, motivating the use of regularization techniques~\cite{chu2003interpolation}.
Following this line of work, we regularize the Nadaraya--Watson estimator by introducing a prior in the form of an additive constant $\lambda > 0$ in the denominator.
This modification prevents the effective local sample size from collapsing in low-density regions and shrinks the estimate toward zero when the data are uninformative.
Irrespective of the above notation, given the observations $x_i$ and the times $t_i$, the final interpolation routine is as follows:
$$
\hat{x}(t) = \frac{\sum_i x_i K(t_i, t)}{\lambda + \sum_i K(t_i, t)},
$$
where $K(a, b) = e^{|a-b|}$ is a simple exponential kernel.
In practice, we set $\lambda := K(t-3, t)$\footnote{Equivalently, the procedure can be interpreted as introducing a globally present pseudo-observation with fixed weight and value zero, acting as a weak prior rather than a geometrically localized data point.
}.
For multivariate time series (and on all intermediate levels), the procedure described above is identically and independently repeated for each channel.

\paragraph{Vector Field}
We also employ a non-standard vector field.
The matrix-based Neural CDE dynamics function from~\cite{kidger2020neural} is inapplicable at the intermediate layers of our model, since it relies on the low-dimensional nature of the input signal. 
In our case, both the input and output signals are vectors of dimension $d$ (the model's hidden dimension), so a vanilla Neural CDE would require $\mathcal{O}(d^3)$ parameters per layer.
Instead, we adopt a \emph{SwiGLU}-based vector field~\cite{shazeer2020glu}. This choice is motivated by several factors:
\begin{itemize}
    \item SwiGLU has become the industry standard activation in modern Transformer architectures~\cite{touvron2023llama,yang2025qwen3}. It offers excellent computational speed, strong representational capacity, and a modest memory footprint.
    \item SwiGLU is an advanced gated activation that is far less susceptible to saturation or dying gradient issues, which are a significant concern to Neural CDE models, as highlighted by recent analyses~\cite{kuleshov2024denots}. 
    \item While the unbounded nature of SwiGLU could, in principle, harm model stability, UFO applies Neural CDEs only over short intervals, so these concerns remain minor. Besides, our use of kernel smoothing alleviates stiffness in the dynamics, further enhancing overall model stability.
\end{itemize}

\subsubsection{Transformer Refining.}
\label{sec:transformer_refining}
As mentioned in the Introduction, we use Encoder-Decoder Transformers~\footnote{Recent work has increasingly focused on decoder-only Transformers~\cite{radford2018improving}. However, for our purposes, the classic Transformer encoder–decoder architecture from~\cite{vaswani2017attention} is better suited to our U-Net design. The cross-attention decoder is more computationally efficient, as it functions primarily for mixing in encoder representations: it injects encoder information into the decoder while processing only the decoder embeddings separately, rather than processing them jointly with the encoder ones.} as the refining layers.
Specifically, the encoder refiners consist of self-attention layers followed by feed-forward SwiGLU layers.
For the decoder, a cross-attention layer is inserted between the self-attention and the feed-forward layers.
The cross-attention takes features from the corresponding encoder layer.

\section{Theory}
\subsection{CDE Rescaling Lipschitzness}
An important consideration when using CDE rescaling is the regularity of the newly generated higher-level trajectories.
These trajectories serve as input to subsequent CDE layers, so their smoothness and Lipschitz properties are critical.

We consider the following simplified setting, abstracting away from the full model notation. 
Let $ \mathbf{x}(t): \mathbb{R} \rightarrow \mathbb{R}^d$ be a continuous input signal that is $ L_x $-Lipschitz continuous, i.e.,
\[
\| \mathbf{x}(t) - \mathbf{x}(s) \| \leq L_x |t - s| \quad \forall\, t,s \in \mathbb{R}.
\]

Let $ \mathbf{f}(\mathbf{x}, \mathbf{z}) : \mathbb{R}^d \times \mathbb{R}^d \rightarrow \mathbb{R}^d$ be the CDE vector field, assumed to be globally Lipschitz with respect to both arguments with constant $ L_f $:
\[
\| \mathbf{f}(\mathbf{x}_1, \mathbf{z}_1) - \mathbf{f}(\mathbf{x}_2, \mathbf{z}_2) \| \leq L_f \bigl( \| \mathbf{x}_1 - \mathbf{x}_2 \| + \| \mathbf{z}_1 - \mathbf{z}_2 \| \bigr) 
\quad \forall\, \mathbf{x}_1,\mathbf{x}_2,\mathbf{z}_1,\mathbf{z}_2.
\]

For a fixed patch duration $ w > 0 $ and any starting time $ t \in \mathbb{R} $, we define the local solution $ \mathbf{z}^{(t)}(\tau) :  \mathbb{R} \rightarrow \mathbb{R}^d$ on the interval $ \tau \in [0, w] $ by
\[
\mathbf{z}^{(t)}(0) = \mathbf{0}, \qquad 
\frac{d}{d\tau} \mathbf{z}^{(t)}(\tau) = \mathbf{f}\bigl( \mathbf{x}(t + \tau),\ \mathbf{z}^{(t)}(\tau) \bigr), \quad \tau \in [0, w].
\]

The higher-level (rescaled) trajectory is then obtained by evaluating at the end of each patch and rescaling time:
\[
\boldsymbol{\Phi}(t \cdot w) := \mathbf{z}^{(t)}(w).
\]

We prove the following result on the Lipschitz regularity of $ \boldsymbol{\Phi} $.

\begin{theorem} \label{th:traj_lipsch}
The map $ t \mapsto \boldsymbol{\Phi}(t \cdot w) $ is Lipschitz continuous with constant
\[
L_\Phi = L_x \cdot \frac{e^{L_f w} - 1}{L_f w},
\]
i.e., the effective Lipschitz constant per unit of the new (coarser) time.
\end{theorem}

The proof is provided in Appendix~\ref{ap:lipschitzness}.

The following corollary trivially follows from the above theorem:

\begin{corollary}
    The Lipschitz constant $ L_\Phi $ tends to $ L_x $ as $ w \rightarrow 0 $.
\end{corollary}

Consequently, for a small enough $ w $, the Encoder trajectories share the same Lipschitz constants on all levels.
\subsection{Patching Regularization}
\label{sec:patching_regularization}
When extracting patches from the input sequence for UFO, we skip all missing observations and construct patches of equal length rather than equal duration. 
This design ensures that all higher-level embeddings carry comparable information content, while also shortening the resulting patch sequences, thereby improving the efficiency and effectiveness of the method compared to an equal-duration approach. 
In principle, however, defining patches by length rather than time span implies that the resulting higher-level representation sequences remain irregular. 
Since we subsequently apply a standard Transformer architecture to the patch embedding sequence, such irregularity could pose a challenge. 
Fortunately, the patching procedure progressively regularizes the observation grid as depth increases, mitigating this concern.
We prove the "patching regularization" effect following a simple theoretical fact:
\begin{lemma} [Patching Regularization] \label{lemma:irreg}
    Consider an irregular sequence, with observations at times~$\mathcal{T} = t_1,t_2\ldots, t_n$, such that the inter-observation gaps are i.i.d. random variables: $\Delta t_i = t_{i+1} - t_i$, with a coefficient of variation\footnote{Note that we use the coefficient of variation here instead of variance. 
    Since the raw inter-observation gap variance depends on the scale of time, which is arbitrary, it makes sense to normalize the variance w.r.t. time scale, which is precisely what the coefficient of variation does.} (standard deviation over mean) of $\alpha \in \mathbb{R}$.
    Also consider another sequence $\mathcal{T}' = t_1', t_2', \ldots, t_m'$, which is created from $\mathcal{T}$ through our patched agregation procedure, i.e. by taking every $s$-th time, where $s \in \mathbb{N}$.
    Then, the inter-observation gaps of~$\mathcal{T}'$, denoted by $\Delta t_i' = t_{i+1} - t_i$ will have a coefficient of variation $\alpha' = \frac{\alpha}{\sqrt{s}}$, i.e. $\sqrt{s}$ times smaller.
\end{lemma}
The proof is provided in Appendix~\ref{ap:irreg_lemma}.
In simpler terms, each higher level is more regular than the previous one.
Consequently, we may use Transformers on the higher levels, treating the sequence as if the observations were uniformly spaced.

\section{Results}
In this section, we will present our main experimental results.
Additional results can be found in Appendix~\ref{ap:results}.
More details on the datasets, baselines, and preprocessing steps used are provided in Appendix~\ref{ap:exp_details}.

\subsection{Experiment Setup}
\label{sec:setup}

\paragraph{Metrics}
In practice, we approximate the predictive distributions using multiple independent samples (typically 100 per test sequence).
We evaluate distributional accuracy using the Continuous Ranked Probability Score (CRPS) \cite{gneiting2007strictly,zhang2024probts}, a strictly proper "lower-is-better" scoring rule.
To average CRPS over heterogenous channels, we first normalize each channel independently, and then take the mean (in contrast to~\cite{wu2025k}); we call the resulting metric Normalized CRPS (NCRPS).
For further details, see~\ref{ap:crps}

\paragraph{Datasets \& Baselines}
We conducted experiments on two benchmarks: regular and irregular.
Each consists of the same 5 widely used time series forecasting datasets: ETTm1, ETTm2, Electricity, Weather, and Traffic.
We compare with 10 modern baselines, including five modern non-differential ones: DeepAR~\cite{salinas2020deepar}, PatchTST~\cite{nie2022time}, LaST~\cite{wang2022learning}, $K^2$VAE~\cite{wu2025k}, TSDiff~\cite{kollovieh2023predict}, and five Neural ODE-based ones: Latent ODE~\cite{rubanova2019latent}, Neural CDE~\cite{kidger2020neural}, Trajectory Flow Matching (TFM)~\cite{zhang2024trajectory}, Structured Linear CDEs (SLiCE)~\cite{walker2025structured}, and an advanced Neural CDE version with kernel interpolation~\cite{serov2026efficient} (NCDE++).

\paragraph{Implementation Details}
Neural CDE is unable to handle highly multivariate datasets, such as Electricity ($\sim\!300$ variates) and Traffic ($\sim\!800$ variates), due to overflow in the matrix-defined vector field and an overly-stiff interpolation regime.
To enrich our benchmark, we replace the matrix vector field with SwiGLU, and use kernel-based interpolation both in NCDE++ and in LatentODE, while keeping the Neural CDE model as-is for comparison.
Since NCRPS is strictly proper, we also use it as the loss for non-autoregressive methods, including the proposed UFO, by sampling several predictions per training example.

In our experiments, the context length (number of visible past observations) is equal to the horizon length (number of future observations to predict) with $T = L = 720$.
Consequently, our benchmark is focused on the difficult task of long-term forecasting~\cite{zhang2024probts}.
We also include a short-term forecasting study that compares UFO with the runner-up, $K^2$VAE on $T=L=24$, in Appendix~\ref{ap:short-term}.

\subsection{Qualitative Results}
\label{sec:qual_res}
The qualitative results, comparing UFO performance to that of all the considered baselines in terms of NCRPS, are given in Table~\ref{tab:main_res}.

\begin{table*}[!t]

\caption{Main results of our experiments on regular and irregular benchmarks.
Each run is performed on three different seeds, and the results are averaged.
In each expression $a +- b$, the first number $a$ represents the average NCRPS $(\downarrow)$, and the second $b$ is the corresponding standard deviation.
The best results are highlighted in bold, the second-best are underlined.
Results may share a 1st or 2nd place, if they are statistically indistinguishable with the p-value of a one-sided t-test being $>0.05$.
The datasets on which Neural CDE was unable to learn due to instability are marked by a cross (\xmark).
}
\label{tab:main_res}

\begin{tabular}{lcccccc}\toprule
Method &ETTm1 &ETTm2 &Electricity &Weather &Traffic \\\midrule
\multicolumn{6}{c}{Regular Benchmarks} \\\midrule
DeepAR &0.584 +- 0.074 &1.286 +- 0.386 &\ul{0.138 +- 0.013} &0.72 +- 0.094 &\ul{0.268 +- 0.038} \\
PatchTST &0.439 +- 0.04 &1.233 +- 0.108 &0.148 +- 0.006 &0.67 +- 0.107 &\ul{0.246 +- 0.004} \\
LaST &0.401 +- 0.008 &1.838 +- 0.389 &0.171 +- 0.01 &0.646 +- 0.031 &0.278 +- 0.009 \\
$K^2$VAE &\ul{0.274 +- 0.001} &\ul{0.595 +- 0.04} &\textbf{0.1 +- 0.002} &0.564 +- 0.019 &0.185 +- 0.004 \\
TSDiff &0.549 +- 0.03 &1.516 +- 0.387 &\ul{0.125 +- 0.003} &\textbf{0.363 +- 0.012} &\ul{0.247 +- 0.012} \\
Latent ODE &0.533 +- 0.018 &2.211 +- 0.08 &0.247 +- 0.02 &0.446 +- 0.027 &0.45 +- 0.02 \\
NCDE &\textbf{0.261 +- 0.003} &\textbf{0.447 +- 0.013} &\xmark &\ul{0.381 +- 0.006} &\xmark \\
TFM &0.604 +- 0.045 &2.297 +- 0.15 &0.302 +- 0.024 &0.912 +- 0.143 &0.556 +- 0.039 \\
SLiCE &0.539 +- 0.003 &1.959 +- 0.117 &0.239 +- 0.004 &1.24 +- 0.024 &0.442 +- 0.0 \\
NCDE++ &0.299 +- 0.001 &\textbf{0.434 +- 0.015} &0.214 +- 0.0 &0.42 +- 0.0 &0.432 +- 0.009 \\
UFO (ours) &\textbf{0.271 +- 0.012} &\textbf{0.429 +- 0.009} &\textbf{0.102 +- 0.005} &\textbf{0.36 +- 0.006} &\textbf{0.184 +- 0.002} \\\midrule
\multicolumn{6}{c}{Irregular Benchmarks} \\\midrule
DeepAR &0.619 +- 0.135 &3.168 +- 0.579 &0.14 +- 0.004 &0.625 +- 0.065 &\ul{0.239 +- 0.016} \\
PatchTST &0.457 +- 0.046 &1.344 +- 0.114 &0.207 +- 0.001 &0.665 +- 0.098 &0.339 +- 0.002 \\
LaST &0.404 +- 0.003 &2.325 +- 0.292 &0.25 +- 0.071 &0.673 +- 0.028 &0.332 +- 0.011 \\
$K^2$VAE &\ul{0.31 +- 0.005} &0.854 +- 0.126 &\ul{0.134 +- 0.001} &0.631 +- 0.01 &\ul{0.252 +- 0.015} \\
TSDiff &0.46 +- 0.052 &1.432 +- 0.646 &\ul{0.152 +- 0.018} &\textbf{0.442 +- 0.028} &0.338 +- 0.027 \\
Latent ODE &0.522 +- 0.015 &2.078 +- 0.114 &0.231 +- 0.001 &\ul{0.533 +- 0.071} &0.451 +- 0.007 \\
NCDE &\textbf{0.29 +- 0.003} &\ul{0.477 +- 0.008} &\xmark &\textbf{0.427 +- 0.01} &\xmark \\
TFM &0.578 +- 0.042 &2.821 +- 0.159 &0.278 +- 0.006 &0.749 +- 0.112 &0.476 +- 0.028 \\
SLiCE &0.529 +- 0.008 &2.053 +- 0.085 &0.242 +- 0.003 &1.273 +- 0.123 &0.438 +- 0.002 \\
NCDE++ &0.326 +- 0.004 &\ul{0.46 +- 0.01} &0.22 +- 0.002 &\ul{0.469 +- 0.008} &0.436 +- 0.012 \\
UFO (ours) &\textbf{0.294 +- 0.003} &\textbf{0.43 +- 0.001} &\textbf{0.104 +- 0.0} &\textbf{0.412 +- 0.009} &\textbf{0.188 +- 0.002} \\
\bottomrule
\end{tabular}

\end{table*}

\paragraph{Regular Benchmark}
On the regular benchmark, we used the datasets as-is.
Our UFO method consistently wins, outperforming or performing on-par with other baselines.
No other baselines achieve first place in all the considered setups.
However, we observe that both $K^2$VAE and Neural CDE achieve competitive quality, catching up to UFO on 2 datasets.
TSDiff is also a strong competitor, performing on-par with the winners on Weather.

\paragraph{Irregular Benchmark}
To make our data irregular, we randomly remove 30\% of days from the data, replacing them with NaNs for irregular-compatible methods or forward-filling them for uniform-only ones, following~\cite{kidger2020neural}.
Removing whole days instead of singular observations follows our blockwise irregularity design, mentioned in the introduction.
UFO's lead is more significant in this case, with K$^2$VAE and other methods lagging behind.

\subsection{Efficiency Evaluation}
\label{sec:performance_evaluation}
To validate our efficiency claims, we compare the inference time of our method against a range of baselines.
Inference efficiency is particularly critical in long-horizon probabilistic forecasting, where models are often queried repeatedly or deployed within real-time decision-making pipelines.
It also reflects training efficiency\footnote{
While training time is also an important consideration, direct comparison across the considered methods is not well-defined, as it is heavily influenced by factors orthogonal to model architecture, including early stopping criteria, learning rate schedules, solver tolerances, and fundamentally different training--inference asymmetries (e.g., TFM).
Inference time, by contrast, more directly reflects architectural parallelism and scalability, which are the primary efficiency claims of this work.
} for most NCDE-based models (including UFO), since inference is essentially a forward pass in this case.

The results are presented in Figure~\ref{fig:ncde_test_time}, and align with our claims and prior knowledge.
LaST achieves first place due to a time-parallel design with a very lightweight backbone, followed by PatchTST, which is slightly heavier due to Attention, but still very effective.
UFO achieves the best running time across all Neural CDE models, performing on-par with the fast $K^2$VAE model.
This concludes the list of fully time-parallel models.
The kernel-based Neural CDEs (NCDE++, Latent ODE) follow close by, but they have lower prediction accuracy.
Then comes the autoregressive DeepAR.
Trajectory Flow Matching shows poor performance because of a non-end-to-end integration routine, coupled with autoregressive sampling; however, it is still rather fast since it uses the "target prediction" optimization proposed by the authors, skipping the costly integration phase.
SLiCE comes even further: it consists of multiple stacked blocks, forced to do inference sequentially. 
The initially proposed associative scan parallelization works only for encoding the input sequence, and cannot be used for autoregressive sampling.
SLiCE is matched by the diffusion-based TSDiff architecture: indeed, diffusion models are known to be computationally-intensive during inference.
Neural CDE concludes the list, being the least effective model.

\begin{figure}
    \centering
    \includegraphics[width=\linewidth]{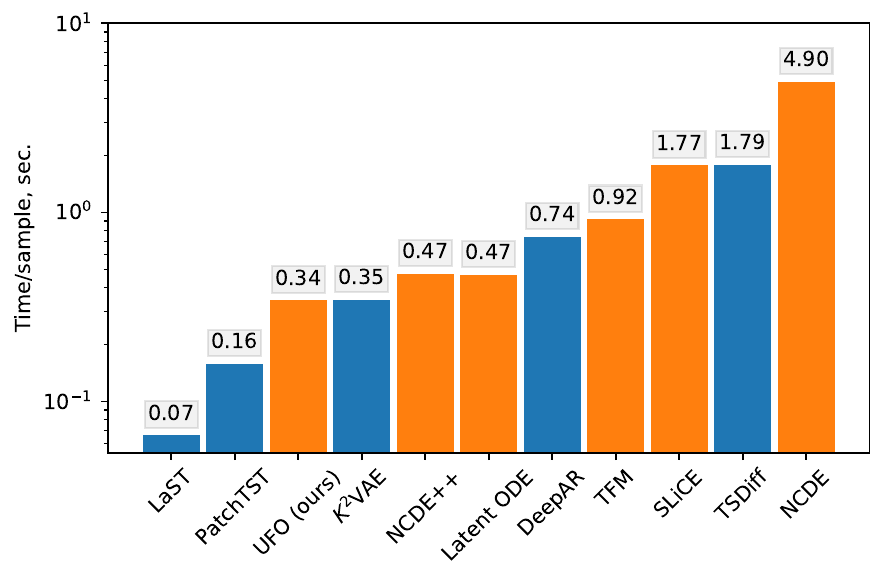}
    \caption{Inference time for all considered models (log scale), performed for a single sample on 6 batches of 32, and then averaged to obtain the time per single-batch sample.
    Classic models are in blue, Neural CDE models are in orange.
    We specify the corresponding time in seconds in a gray rectangle above each bar.
    }
    \label{fig:ncde_test_time}
    \Description{
    This is a bar plot of the inference times of the models from our benchmark.
    }
\end{figure}

\subsection{Sensitivity Analysis}
\label{sec:sensitivity_analysis}
To validate whether our UFO model indeed maintains a global receptive field, paying attention to all input tokens irrespective of their position, we perform a sensitivity analysis.
We measure the average norm of gradients of outputs w.r.t. each input position.
The sensitivity is measured for 4 trained models: Deep AR, PatchTST, Neural CDE, and UFO.
Additionally, we measure the correlation between the input observation number and the log-norm of gradients, to determine whether a model uniformly accounts for all observations, or is more sensitive to older/newer ones.
The results are presented in Figure~\ref{fig:sensitivity_grads}.

\begin{figure*}[tb]
    \centering
    \includegraphics[width=\textwidth]{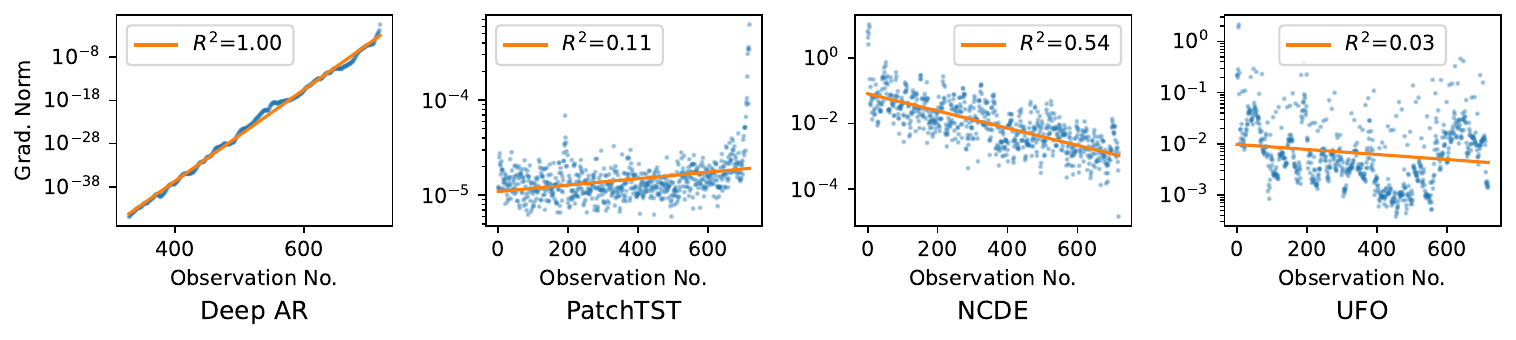}
    \caption{Sensitivity analysis w.r.t. input observation number. Due to limited precision, Deep AR gradients are exactly zero at earlier positions, so we cannot visualize them in log-scale.}
    \label{fig:sensitivity_grads}
    \Description{This image shows scatter plots of average output gradient norms w.r.t. each input position.
    The DeepAR model shows a direct line with $1.0$ squared correlation.
    PatchTST shows more-or-less random observations, with a horizontal line (squared correlation is $0.11$).
    NCDE shows a decreasing trend, with observations located roughly around the line (squared correlation is $0.54$).
    Finally, UFO shows a close to random distribution, with observations appearing at large distances from a horizontal line ($R^2 = 0.03$).
    }
\end{figure*}

Deep AR results are to be expected and follow from the nature of recurrent neural networks.
Indeed, these networks are well known for their "forgetting" effect, with older observations receiving exponentially smaller gradients~\cite{hochreiter1997long}.
The corresponding $R^2$ correlation coefficient between log-gradient norm and observation number is exactly $1$.
The sensitivity graph of PatchTST is also unsurprising, with gradients spread more or less evenly across the sequence, with $R^2 = 0.12$.
Transformers uniformly attend to all inputs by design.

The results are surprising for the Neural CDE model.
The sensitivity \emph{decreases} with observation number: the model is sensitive to older observations, with $R^2 = 0.51$.
Although unintuitive at first, we argue that the finding aligns with prior theoretical contributions.
Specifically, we recite an informal version of Theorem 3.1 from~\cite{kuleshov2024denots}:
\begin{theorem*}[Informal version, from~\cite{kuleshov2024denots}]
    The Lipschitz constant of Neural CDE w.r.t. the input signal on an interval of length $\tau$ is proportional to $e^\tau$.
\end{theorem*}
Consequently, any observation may only influence the hidden state by no more than $\sim e^\tau$, where $\tau$ represents time until the end of the sequence.
For the most recent observations, $\tau$ is smaller, so they have less effect on the final hidden state.
The paper~\cite{kuleshov2024denots} successfully addresses this effect by scaling time, but this also increases computational cost.
As demonstrated in Figure~\ref{fig:sensitivity_grads}, our approach does not adhere to the above rule, with $R^2 = 0$.
UFO adopts the Transformer architecture to evenly spread the gradients throughout the sequence, while performing several times faster than Neural CDE due to a time-parallel architecture.

\subsection{Resamplers Ablation}
To demonstrate the utility of our main contribution, we compare different Up- and Downsampling approaches within the U-Net Transformer, namely RNNs, convolutions, and Neural CDEs, in both regular and irregular data settings.
We evaluate on the ETT datasets, which naturally contain whole-day gaps and therefore provide a suitable testbed for assessing robustness to data irregularity.

\begin{table}[!t]

\caption{Kernel ablation on the ETT datasets, regular and irregular.
As with Table~\ref{tab:main_res}, the resulting NCPRS ($\downarrow$) are averaged over three runs, presented as "mean +- std``.
For irregular data, we also report the increase in NCRPS relative to the performance of regular data.
}
\label{tab:kernel_abl}

\begin{tabular*}{\linewidth}{@{\extracolsep{\fill}}lll}
\toprule
Method & ETTm1 & ETTm2 \\ 
\midrule\addlinespace[2.5pt]
\multicolumn{3}{c}{Regular data} \\
\midrule
UF-RNN & 0.269 +- 0.003 & 0.427 +- 0.001 \\
UF-Conv & 0.266 +- 0.005 & 0.432 +- 0.004 \\
UFO & 0.271 +- 0.012 & 0.429 +- 0.009 \\
\midrule\addlinespace[2.5pt]
\multicolumn{3}{c}{Irregular data} \\
\midrule
UF-RNN & 0.317 +- 0.013 (+0.048) & 0.444 +- 0.001 (+0.015) \\
UF-Conv & 0.315 +- 0.004 (+0.049) & 0.5 +- 0.008 (+0.068)\\
UFO & \textbf{0.294 +- 0.003} (+0.023) & \textbf{0.43 +- 0.001} (+0.001) \\
\bottomrule
\end{tabular*}

\end{table}

The results are summarized in Table~\ref{tab:kernel_abl}.
On regular data, we observe no statistically significant differences between RNN-, convolution-~, and CDE-based resampling techniques, which is consistent with our claims.
In contrast, on the irregular benchmark, UFO clearly outperforms the alternatives.
To facilitate comparison, we additionally report the performance change when moving from regular to irregular data.
On ETTm1, the degradation in NCRPS for UFO is approximately half that of UF-RNN and UF-Conv, while on ETTm2, it is negligible and within one standard deviation.
These results indicate that using Neural CDEs as up/down-sampling kernels substantially improves the ability of the U-Net Transformer to handle irregular time series.

\subsection{Irregularity Encoding}
\label{sec:irregularity_encoding}
We further hypothesize that Neural CDEs encode observation irregularity into patch embeddings.
To evaluate this hypothesis, we conduct the following study.
Using pre-trained UFO, UF-RNN, and UF-Conv models on the irregular variants of ETTm1 and ETTm2, we extract patch embeddings produced by the bottom Neural CDE downsampling layer.
Based solely on these embeddings, we then formulate a binary classification task that distinguishes between patches with and without missing observations.
Classification is performed using a logistic regression model with balanced class weights to account for class imbalance.
Prediction quality is measured using the F1 score.

\begin{table}[tb]
    \centering
    \caption{The patch irregularity classification study. The metric (F1 score, $\uparrow$) is averaged over three runs, each cell presents "mean $\pm$ std".}
    \label{tab:irreg_clf_abl}
    \begin{tabular}{lccc}\toprule
    Method &ETTm1 &ETTm2 \\\midrule
    UF-RNN &0.182 $\pm$ 0.006 &0.097 $\pm$ 0.01 \\
    UF-Conv &0.205 $\pm$ 0.024 &0.093 $\pm$ 0.024 \\
    UFO & \textbf{0.781 $\mathbf{\pm}$ 0.144} & \textbf{0.363 $\mathbf{\pm}$ 0.147} \\
    \bottomrule
    \end{tabular}
\end{table}

The results are reported in Table~\ref{tab:irreg_clf_abl}.
They clearly demonstrate that the Neural CDE-based UFO model learns embeddings that capture observation irregularity, whereas UF-RNN and UF-Conv fail to do so.

\section{Related Works}
\label{sec:review}
We separate the review of regular and irregular data models, focusing on the most relevant architectures: U-Nets, Transformers and various versions of Neural CDEs.
We provide only a brief recollection of prior works in the main text,
an extended version of this review can be found in Appendix~\ref{sec:ext_review}.

\subsection{Classic models (Regular Data)}

\paragraph{U-Net}
Adapting U-Net~\cite{ronneberger2015u} to time series poses challenges, as standard convolutions and pooling may be suboptimal. Kernel U-Net~\cite{you2024kernel} introduces flexible temporal downscaling (\emph{kernels}) like RNNs and linear projections, but retains one-to-one skip connections, limiting context. Our approach generalizes this via cross-attention.

\paragraph{Transformers}
There is plenty of work on Transformers for time series~\cite{li2019enhancing,zhou2021informer,madhusudhanan2021yformer,wu2021autoformer}; most of these methods were later outperformed by simpler patching-based models like PatchTST~\cite{nie2022time}. 
However, vanilla attention is permutation-invariant~\cite{zeng2023transformers}, so architectures encoding temporal dynamics can help. 
Neural ODE-based models complement attention, handling irregular and gapped data-a capability prior U-Net-like Transformers generally lack.

\subsection{Neural CDEs (Irregular Data)}

\paragraph{Neural CDE}
The original Neural ODEs~\cite{chen2018neural,rubanova2019latent} integrate dynamics only between observations.
Neural CDEs~\cite{kidger2020neural} interpolate inputs for continuous-time integration, but standard CDEs scale cubically in parameter count and are inherently sequential. Structured Linear CDE (SLiCE)~\cite{walker2025slices} reduces the number of parameters and allows parallelization, yet its performance remains suboptimal compared to our method: associative scans are poorly suited for autoregressive forecasting.

\paragraph{Neural RDE and patching}
Neural RDEs~\cite{morrill2021neural} and the later Neural Log-ODEs~\cite{walker2024log} apply log-signature transforms to input patches, reducing sequence length.
However, the dimensionality of the resulting patches grows exponentially with the number of input variates, limiting scalability.

\paragraph{Neural CDE Transformers}
Hybrid CDE-attention models~\cite{chen2023contiformer,jhin2024attentive} often offer good representation capabilities, but incur higher computational costs than traditional NCDEs. 
CDEs for temporal point processes~\cite{chen2020neural,song2024decoupled,xiao2024ivp} assume weak correlations and are memory-intensive.
UFO balances these extremes: it is fully parallelizable, memory-efficient, and captures long-range dependencies, making it highly effective for irregular, long-horizon, multivariate forecasting.

\section{Conclusion}
We introduced UFO (U-Former ODE), a novel architecture for probabilistic multivariate time series forecasting, particularly suited for irregularly sampled data. UFO integrates the strengths of three powerful paradigms: the parallelizable hierarchical structure of U-Nets, the global receptive field of Transformers, and the continuous-time modeling capability of Neural ODEs. This combination allows UFO to overcome key limitations of existing methods, such as the sequential bottleneck of Neural CDEs, the limited receptive field of pure convolutional models, and the permutation invariance of standard attention mechanisms in time series.

Our experimental results across five standard benchmarks—on both regular and irregular time series—demonstrate that UFO consistently outperforms nine state-of-the-art baselines in predictive quality, while also being 15 times faster than classical Neural CDEs. Notably, UFO maintains stable performance on long, high-dimensional sequences and shows strong robustness to irregular sampling patterns, including challenging block-wise missing data --- again making a step forward over existing Neural CDE models.

In summary, UFO sets a new standard for fast, accurate, and scalable irregular time series forecasting, bridging the gap between continuous-time dynamics and modern deep learning architectures.

\bibliographystyle{ACM-Reference-Format}
\bibliography{ref}

@inproceedings{zhou2021informer,
  title={Informer: Beyond efficient transformer for long sequence time-series forecasting},
  author={Zhou, Haoyi and Zhang, Shanghang and Peng, Jieqi and Zhang, Shuai and Li, Jianxin and Xiong, Hui and Zhang, Wancai},
  booktitle={Proceedings of the AAAI conference on artificial intelligence},
  volume={35},
  number={12},
  pages={11106--11115},
  year={2021}
}

@article{nie2022time,
  title={A Time Series is Worth 64Words: Long-term Forecasting with Transformers},
  author={Nie, Y},
  journal={arXiv preprint arXiv:2211.14730},
  year={2022}
}

@article{madhusudhanan2021yformer,
  title={Yformer: {U}-net inspired transformer architecture for far horizon time series forecasting},
  author={Madhusudhanan, Kiran and Burchert, Johannes and Duong-Trung, Nghia and Born, Stefan and Schmidt-Thieme, Lars},
  journal={arXiv preprint arXiv:2110.08255},
  year={2021}
}

@inproceedings{you2024kernel,
  title={Kernel-{U}-net: Multivariate time series forecasting using custom kernels},
  author={You, Jiang and Cela, Arben and Natowicz, Ren{\'e} and Ouanounou, Jacob and Siarry, Patrick},
  booktitle={2024 International Conference on INnovations in Intelligent SysTems and Applications (INISTA)},
  pages={1--8},
  year={2024},
  organization={IEEE}
}

@inproceedings{hatamizadeh2022unetr,
  title={Unetr: Transformers for 3d medical image segmentation},
  author={Hatamizadeh, Ali and Tang, Yucheng and Nath, Vishwesh and Yang, Dong and Myronenko, Andriy and Landman, Bennett and Roth, Holger R and Xu, Daguang},
  booktitle={Proceedings of the IEEE/CVF winter conference on applications of computer vision},
  pages={574--584},
  year={2022}
}

@article{chen2018neural,
  title={Neural ordinary differential equations},
  author={Chen, Ricky TQ and Rubanova, Yulia and Bettencourt, Jesse and Duvenaud, David K},
  journal={Advances in neural information processing systems},
  volume={31},
  year={2018}
}

@article{rubanova2019latent,
  title={Latent ordinary differential equations for irregularly-sampled time series},
  author={Rubanova, Yulia and Chen, Ricky TQ and Duvenaud, David K},
  journal={Advances in neural information processing systems},
  volume={32},
  year={2019}
}

@article{kidger2020neural,
  title={Neural controlled differential equations for irregular time series},
  author={Kidger, Patrick and Morrill, James and Foster, James and Lyons, Terry},
  journal={Advances in neural information processing systems},
  volume={33},
  pages={6696--6707},
  year={2020}
}

@inproceedings{morrill2021neural,
  title={Neural rough differential equations for long time series},
  author={Morrill, James and Salvi, Cristopher and Kidger, Patrick and Foster, James},
  booktitle={International Conference on Machine Learning},
  pages={7829--7838},
  year={2021},
  organization={PMLR}
}

@inproceedings{zeng2023transformers,
  title={Are transformers effective for time series forecasting?},
  author={Zeng, Ailing and Chen, Muxi and Zhang, Lei and Xu, Qiang},
  booktitle={Proceedings of the AAAI conference on artificial intelligence},
  volume={37},
  number={9},
  pages={11121--11128},
  year={2023}
}

@article{kuleshov2024denots,
  title={DeNOTS: Stable Deep Neural ODEs for Time Series},
  author={Kuleshov, Ilya and Romanenkova, Evgenia and Zhuzhel, Vladislav and Boeva, Galina and Vorsin, Evgeni and Zaytsev, Alexey},
  journal={arXiv preprint arXiv:2408.08055},
  year={2024}
}

@article{khayati2020mind,
  title={Mind the gap},
  author={Khayati, Mourad and Lerner, Alberto and Tymchenko, Zakhar and Cudr{\'e}-Mauroux, Philippe},
  journal={Proceedings of the VLDB Endowment},
  volume={13},
  pages={768--782},
  year={2020}
}

@article{chung2014empirical,
  title={Empirical evaluation of gated recurrent neural networks on sequence modeling},
  author={Chung, Junyoung and Gulcehre, Caglar and Cho, KyungHyun and Bengio, Yoshua},
  journal={arXiv preprint arXiv:1412.3555},
  year={2014}
}

@article{vaswani2017attention,
  title={Attention is all you need},
  author={Vaswani, Ashish and Shazeer, Noam and Parmar, Niki and Uszkoreit, Jakob and Jones, Llion and Gomez, Aidan N and Kaiser, {\L}ukasz and Polosukhin, Illia},
  journal={Advances in neural information processing systems},
  volume={30},
  year={2017}
}

@article{walker2024log,
  title={Log neural controlled differential equations: The lie brackets make a difference},
  author={Walker, Benjamin and McLeod, Andrew D and Qin, Tiexin and Cheng, Yichuan and Li, Haoliang and Lyons, Terry},
  journal={arXiv preprint arXiv:2402.18512},
  year={2024}
}

@article{luo2016understanding,
  title={Understanding the effective receptive field in deep convolutional neural networks},
  author={Luo, Wenjie and Li, Yujia and Urtasun, Raquel and Zemel, Richard},
  journal={Advances in neural information processing systems},
  volume={29},
  year={2016}
}

@article{chen2023contiformer,
  title={Contiformer: Continuous-time transformer for irregular time series modeling},
  author={Chen, Yuqi and Ren, Kan and Wang, Yansen and Fang, Yuchen and Sun, Weiwei and Li, Dongsheng},
  journal={Advances in Neural Information Processing Systems},
  volume={36},
  pages={47143--47175},
  year={2023}
}

@article{chen2020neural,
  title={Neural spatio-temporal point processes},
  author={Chen, Ricky TQ and Amos, Brandon and Nickel, Maximilian},
  journal={arXiv preprint arXiv:2011.04583},
  year={2020}
}

@article{song2024decoupled,
  title={Decoupled marked temporal point process using neural ordinary differential equations},
  author={Song, Yujee and Lee, Donghyun and Meng, Rui and Kim, Won Hwa},
  journal={arXiv preprint arXiv:2406.06149},
  year={2024}
}

@article{shaker2024unetr++,
  title={UNETR++: delving into efficient and accurate 3D medical image segmentation},
  author={Shaker, Abdelrahman and Maaz, Muhammad and Rasheed, Hanoona and Khan, Salman and Yang, Ming-Hsuan and Khan, Fahad Shahbaz},
  journal={IEEE Transactions on Medical Imaging},
  volume={43},
  number={9},
  pages={3377--3390},
  year={2024},
  publisher={IEEE}
}

@article{gluonts_jmlr,
  author  = {Alexander Alexandrov and Konstantinos Benidis and Michael Bohlke-Schneider
    and Valentin Flunkert and Jan Gasthaus and Tim Januschowski and Danielle C. Maddix
    and Syama Rangapuram and David Salinas and Jasper Schulz and Lorenzo Stella and
    Ali Caner Türkmen and Yuyang Wang},
  title   = {{GluonTS: Probabilistic and Neural Time Series Modeling in Python}},
  journal = {Journal of Machine Learning Research},
  year    = {2020},
  volume  = {21},
  number  = {116},
  pages   = {1-6},
  url     = {http://jmlr.org/papers/v21/19-820.html}
}

@article{dosovitskiy2020image,
  title={An image is worth 16x16 words: Transformers for image recognition at scale},
  author={Dosovitskiy, Alexey},
  journal={arXiv preprint arXiv:2010.11929},
  year={2020}
}

@article{yun2019graph,
  title={Graph transformer networks},
  author={Yun, Seongjun and Jeong, Minbyul and Kim, Raehyun and Kang, Jaewoo and Kim, Hyunwoo J},
  journal={Advances in neural information processing systems},
  volume={32},
  year={2019}
}

@article{zhao2024recommender,
  title={Recommender systems in the era of large language models (llms)},
  author={Zhao, Zihuai and Fan, Wenqi and Li, Jiatong and Liu, Yunqing and Mei, Xiaowei and Wang, Yiqi and Wen, Zhen and Wang, Fei and Zhao, Xiangyu and Tang, Jiliang and others},
  journal={IEEE Transactions on Knowledge and Data Engineering},
  volume={36},
  number={11},
  pages={6889--6907},
  year={2024},
  publisher={IEEE}
}

@article{wu2021autoformer,
  title={Autoformer: Decomposition transformers with auto-correlation for long-term series forecasting},
  author={Wu, Haixu and Xu, Jiehui and Wang, Jianmin and Long, Mingsheng},
  journal={Advances in neural information processing systems},
  volume={34},
  pages={22419--22430},
  year={2021}
}

@article{li2019enhancing,
  title={Enhancing the locality and breaking the memory bottleneck of transformer on time series forecasting},
  author={Li, Shiyang and Jin, Xiaoyong and Xuan, Yao and Zhou, Xiyou and Chen, Wenhu and Wang, Yu-Xiang and Yan, Xifeng},
  journal={Advances in neural information processing systems},
  volume={32},
  year={2019}
}

@inproceedings{ronneberger2015u,
  title={U-net: Convolutional networks for biomedical image segmentation},
  author={Ronneberger, Olaf and Fischer, Philipp and Brox, Thomas},
  booktitle={International Conference on Medical image computing and computer-assisted intervention},
  pages={234--241},
  year={2015},
  organization={Springer}
}

@article{jhin2024attentive,
  title={Attentive neural controlled differential equations for time-series classification and forecasting},
  author={Jhin, Sheo Yon and Shin, Heejoo and Kim, Sujie and Hong, Seoyoung and Jo, Minju and Park, Solhee and Park, Noseong and Lee, Seungbeom and Maeng, Hwiyoung and Jeon, Seungmin},
  journal={Knowledge and Information Systems},
  volume={66},
  number={3},
  pages={1885--1915},
  year={2024},
  publisher={Springer}
}

@misc{walker2025slices,
  title        = {Structured Linear CDEs: Maximally Expressive and Parallel-in-Time Sequence Models},
  author       = {Walker, Benjamin and Yang, Lingyi and Muca Cirone, Nicola and Salvi, Cristopher and Lyons, Terry},
  year         = {2025},
  month        = {May},
  url          = {https://arxiv.org/abs/2505.17761},
}

@inproceedings{gu2024mamba,
  title={Mamba: Linear-time sequence modeling with selective state spaces},
  author={Gu, Albert and Dao, Tri},
  booktitle={First conference on language modeling},
  year={2024}
}

@inproceedings{lai2018modeling,
  title={Modeling long-and short-term temporal patterns with deep neural networks},
  author={Lai, Guokun and Chang, Wei-Cheng and Yang, Yiming and Liu, Hanxiao},
  booktitle={The 41st international ACM SIGIR conference on research \& development in information retrieval},
  pages={95--104},
  year={2018}
}

@article{PhysioNet-challenge-2019-1.0.0,
  author = {Reyna, Matthew and Josef, Chris and Jeter, Russell and Shashikumar, Supreeth and Moody, Benjamin and Westover, M. Brandon and Sharma, Ashish and Nemati, Shamim and Clifford, Gari D.},
  title = {{Early Prediction of Sepsis from Clinical Data: The PhysioNet/Computing in Cardiology Challenge 2019}},
  journal = {{PhysioNet}},
  year = {2019},
  month = aug,
  note = {Version 1.0.0},
  doi = {10.13026/v64v-d857},
  url = {https://doi.org/10.13026/v64v-d857}
}

@inproceedings{xiao2024ivp,
  title={IVP-VAE: modeling EHR time series with initial value problem solvers},
  author={Xiao, Jingge and Basso, Leonie and Nejdl, Wolfgang and Ganguly, Niloy and Sikdar, Sandipan},
  booktitle={Proceedings of the AAAI Conference on Artificial Intelligence},
  volume={38},
  number={14},
  pages={16023--16031},
  year={2024}
}

@article{gneiting2007strictly,
  title={Strictly proper scoring rules, prediction, and estimation},
  author={Gneiting, Tilmann and Raftery, Adrian E},
  journal={Journal of the American statistical Association},
  volume={102},
  number={477},
  pages={359--378},
  year={2007},
  publisher={Taylor \& Francis}
}

@inproceedings{kim2021reversible,
  title={Reversible instance normalization for accurate time-series forecasting against distribution shift},
  author={Kim, Taesung and Kim, Jinhee and Tae, Yunwon and Park, Cheonbok and Choi, Jang-Ho and Choo, Jaegul},
  booktitle={International conference on learning representations},
  year={2021}
}

@article{shazeer2020glu,
  title={{GLU} variants improve transformer},
  author={Shazeer, Noam},
  journal={arXiv preprint arXiv:2002.05202},
  year={2020}
}

@article{touvron2023llama,
  title={Llama: Open and efficient foundation language models},
  author={Touvron, Hugo and Lavril, Thibaut and Izacard, Gautier and Martinet, Xavier and Lachaux, Marie-Anne and Lacroix, Timoth{\'e}e and Rozi{\`e}re, Baptiste and Goyal, Naman and Hambro, Eric and Azhar, Faisal and others},
  journal={arXiv preprint arXiv:2302.13971},
  year={2023}
}

@article{yang2025qwen3,
  title={Qwen3 technical report},
  author={Yang, An and Li, Anfeng and Yang, Baosong and Zhang, Beichen and Hui, Binyuan and Zheng, Bo and Yu, Bowen and Gao, Chang and Huang, Chengen and Lv, Chenxu and others},
  journal={arXiv preprint arXiv:2505.09388},
  year={2025}
}

@article{salinas2020deepar,
  title={DeepAR: Probabilistic forecasting with autoregressive recurrent networks},
  author={Salinas, David and Flunkert, Valentin and Gasthaus, Jan and Januschowski, Tim},
  journal={International journal of forecasting},
  volume={36},
  number={3},
  pages={1181--1191},
  year={2020},
  publisher={Elsevier}
}

@article{wang2022learning,
  title={Learning latent seasonal-trend representations for time series forecasting},
  author={Wang, Zhiyuan and Xu, Xovee and Zhang, Weifeng and Trajcevski, Goce and Zhong, Ting and Zhou, Fan},
  journal={Advances in Neural Information Processing Systems},
  volume={35},
  pages={38775--38787},
  year={2022}
}

@article{wu2025k,
  title={{$K^2$}VAE: A Koopman-Kalman Enhanced Variational AutoEncoder for Probabilistic Time Series Forecasting},
  author={Wu, Xingjian and Qiu, Xiangfei and Gao, Hongfan and Hu, Jilin and Yang, Bin and Guo, Chenjuan},
  journal={arXiv preprint arXiv:2505.23017},
  year={2025}
}

@article{kollovieh2023predict,
  title={Predict, refine, synthesize: Self-guiding diffusion models for probabilistic time series forecasting},
  author={Kollovieh, Marcel and Ansari, Abdul Fatir and Bohlke-Schneider, Michael and Zschiegner, Jasper and Wang, Hao and Wang, Yuyang Bernie},
  journal={Advances in Neural Information Processing Systems},
  volume={36},
  pages={28341--28364},
  year={2023}
}

@article{zhang2024trajectory,
  title={Trajectory flow matching with applications to clinical time series modelling},
  author={Zhang, Xi Nicole and Pu, Yuan and Kawamura, Yuki and Loza, Andrew and Bengio, Yoshua and Shung, Dennis and Tong, Alexander},
  journal={Advances in Neural Information Processing Systems},
  volume={37},
  pages={107198--107224},
  year={2024}
}

@article{walker2025structured,
  title={Structured Linear CDEs: Maximally Expressive and Parallel-in-Time Sequence Models},
  author={Walker, Benjamin and Yang, Lingyi and Cirone, Nicola Muca and Salvi, Cristopher and Lyons, Terry},
  journal={arXiv preprint arXiv:2505.17761},
  year={2025}
}

@misc{serov2026efficient,
      title={Efficient Neural Controlled Differential Equations via Attentive Kernel Smoothing}, 
      author={Egor Serov and Ilya Kuleshov and Alexey Zaytsev},
      year={2026},
      eprint={2602.02157},
      archivePrefix={arXiv},
      primaryClass={cs.LG},
      url={https://arxiv.org/abs/2602.02157}, 
}

@article{zhang2024probts,
  title={ProbTS: Benchmarking point and distributional forecasting across diverse prediction horizons},
  author={Zhang, Jiawen and Wen, Xumeng and Zhang, Zhenwei and Zheng, Shun and Li, Jia and Bian, Jiang},
  journal={Advances in Neural Information Processing Systems},
  volume={37},
  pages={48045--48082},
  year={2024}
}

@article{morrill2021online,
  title={Neural controlled differential equations for online prediction tasks},
  author={Morrill, James and Kidger, Patrick and Yang, Lingyi and Lyons, Terry},
  journal={arXiv preprint arXiv:2106.11028},
  year={2021}
}

@article{chu2003interpolation,
  title={An interpolation method for adapting to sparse design in multivariate nonparametric regression},
  author={Chu, CK and Deng, Wen-Shuenn},
  journal={Journal of statistical planning and inference},
  volume={116},
  number={1},
  pages={91--111},
  year={2003},
  publisher={Elsevier}
}

@article{radford2018improving,
  title={Improving language understanding by generative pre-training},
  author={Radford, Alec and Narasimhan, Karthik and Salimans, Tim and Sutskever, Ilya and others},
  year={2018},
  publisher={San Francisco, CA, USA}
}

@article{hochreiter1997long,
  title={Long short-term memory},
  author={Hochreiter, Sepp and Schmidhuber, J{\"u}rgen},
  journal={Neural computation},
  volume={9},
  number={8},
  pages={1735--1780},
  year={1997},
  publisher={MIT press}
}

@article{gneiting2005calibrated,
  title={Calibrated probabilistic forecasting using ensemble model output statistics and minimum CRPS estimation},
  author={Gneiting, Tilmann and Raftery, Adrian E and Westveld III, Anton H and Goldman, Tom},
  journal={Monthly weather review},
  volume={133},
  number={5},
  pages={1098--1118},
  year={2005}
}

\appendix

\section{Related Works (Extended)}
\label{sec:ext_review}
\subsection{Classic models (Regular Data)}

\paragraph{U-Net}
U-Net was originally introduced as a hierarchical fully convolutional architecture for biomedical image segmentation~\cite{ronneberger2015u}. 
Subsequent work extended this paradigm by incorporating Transformer components. 
UNETR~\cite{hatamizadeh2022unetr} proposed a hybrid Transformer--U-Net architecture, and UNETR++~\cite{shaker2024unetr++} further refined this design with architectural and training improvements. 
However, these methods remain firmly rooted in the computer vision domain, where convolutional inductive biases and fixed-resolution downsampling/upsampling operations are natural and effective.

In contrast, adapting U-Net-style architectures to time series raises different challenges. 
Kernel U-Net~\cite{you2024kernel} argues that standard convolutions and pooling operations are suboptimal for temporal data and proposes more flexible temporal downscaling mechanisms (termed \emph{kernels}), including RNNs and linear projections. 
While this approach broadens the design space, Kernel U-Net retains one-to-one skip connections, such that each decoder patch only accesses information from a single encoder patch. 
This limits the model’s ability to aggregate long-range temporal context.
Instead, our approach follows a more general paradigm, inspired by Informer and PatchTST (described below), in which the decoder employs attention mechanisms to integrate information from all encoder patches, enabling richer temporal interactions.

\paragraph{Transformers}
Transformer models for time series forecasting have been studied extensively.
LogTrans~\cite{li2019enhancing} integrates convolutional operations into attention to capture local patterns.
Informer~\cite{zhou2021informer} employs encoder downsampling and sparse attention for efficiency and accuracy.
Y-Former~\cite{madhusudhanan2021yformer} adopts a U-Net-like hierarchical design in encoder and decoder.
Autoformer~\cite{wu2021autoformer} introduces autocorrelation-based attention to model periodicity explicitly.

Despite these advances, many of these approaches were later outperformed by PatchTST~\cite{nie2022time}, which shows that a simpler patching-based Transformer is highly competitive.
However, recent studies question the suitability of vanilla attention for time series modeling.
Attention is permutation-invariant and uses positional encodings, yet empirical evidence indicates that attention-based models may still struggle to capture temporal order fully.
In particular,~\cite{zeng2023transformers} demonstrates that Transformers can underemphasize sequential dependencies, potentially weakening the representation of causal structure fundamental to time-domain data.
Consequently, Transformer-based models often benefit from being combined with architectures that explicitly encode temporal dynamics.

Neural ODE-based models are closely aligned with the continuous-time nature of time series and explicitly model temporal evolution, making them a natural complement to attention mechanisms.
UFO inherits these properties and can naturally handle irregularly sampled and gapped data, a capability that prior U-Net-like Transformer architectures generally lack.

\subsection{Neural CDEs (Irregular Data)}

\paragraph{Neural CDE}
Neural Ordinary Differential Equations (Neural ODEs) were introduced in~\cite{chen2018neural}, including early applications to time series modeling.
The core idea is to define the derivative of a hidden state using a neural network and integrate it over time.
This framework was extended in~\cite{rubanova2019latent} to latent-variable models via a Neural ODE-based variational autoencoder.
However, both approaches integrate the dynamics only between discrete observations and rely on RNN-style updates at observation times.

Neural Controlled Differential Equations (Neural CDEs)~\cite{kidger2020neural} resolve this limitation by interpolating the input signal (e.g., using cubic splines) and integrating the dynamics end-to-end in continuous time.
While theoretically elegant, the original Neural CDE formulation requires the vector field to output a matrix, leading to cubic parameter growth in the final layer.
This makes deep or stacked Neural CDE architectures impractical.
Moreover, standard Neural CDEs are inherently sequential and cannot be parallelized along the temporal dimension.

Recent work introduces structural constraints to address these issues.
SLiCE~\cite{walker2025slices} proposes a structured linear control matrix that significantly reduces parameter count and enables parallel execution via associative scans, drawing inspiration from state-space models such as Mamba~\cite{gu2024mamba}.
We include SLiCE in our benchmark comparison and find that, despite these optimizations, it delivers suboptimal predictive performance relative to our approach.

\paragraph{Neural RDE and patching}
Neural Rough Differential Equations (Neural RDEs)~\cite{morrill2021neural} address long-horizon time series by applying signature transforms to patches of the input sequence, thereby reducing the effective sequence length.
However, the dimensionality of signature features grows exponentially with the input dimension, as acknowledged by the authors, rendering the approach infeasible for highly multivariate time series (e.g., more than 10--20 variables).
Subsequent work~\cite{walker2024log} improves the efficiency of signature computation, but scalability remains a limiting factor.

\paragraph{Neural CDE Transformers}
Several recent works explore hybrid architectures that combine differential equation models with attention mechanisms.
Attentive Neural CDE (ANCDE)~\cite{jhin2024attentive} employs two coupled Neural CDEs, where one generates attention weights that modulate the input path processed by the second.
ContiFormer~\cite{chen2023contiformer} proposes a continuous-time reformulation of the attention mechanism, but at a prohibitive computational cost, exceeding that of standard Neural CDEs.

Other approaches apply Neural CDEs to temporal point processes (TPPs).
For example,~\cite{chen2020neural} integrates each event independently while using attention to aggregate information from past events.
Similar decoupled integration strategies appear in~\cite{song2024decoupled,xiao2024ivp}.
While effective for event-based data, such methods are ill-suited for time series forecasting for two reasons.
First, they implicitly assume weak correlations between observations, which rarely holds in continuous-valued time series.
Second, they incur prohibitive memory and computation costs for long sequences: the number of ODEs they need to solve is proportinal to sequence length.
In a way, the patching approach of UFO is a compromise between the end-to-end integration of Neural CDE, which takes too much time and forgets past observations, and the per-observation integration of TPP methods, which discard neighbourhood correlations and are very memory-intensive.

More broadly, most Neural CDE- and RDE-based methods primarily target time series classification rather than forecasting.
In contrast, UFO is explicitly designed for forecasting and offers several practical advantages:
(a) it is fully parallelizable along the temporal dimension, achieving runtimes comparable to non-ODE models;
(b) its memory footprint matches that of a single feed-forward layer; and
(c) it provides Transformer-like long-range dependency modeling.
Together, these properties make UFO particularly well-suited for irregular, long-horizon, multivariate time series forecasting, where it substantially outperforms prior methods in both efficiency and accuracy.

\section{Proofs}

\subsection{Lipschitzness Theorem}
\label{ap:lipschitzness}

\begin{theorem*} [\ref{th:traj_lipsch}]
The map $t \mapsto \mathbf{\Phi}(t \cdot w)$ is Lipschitz continuous with constant
\[
L_\Phi = L_x \cdot \frac{e^{L_f w} - 1}{L_f w}.
\]
i.e. the effective Lipschitz constant per unit of the new (coarser) time.
\end{theorem*}
\begin{proof}
Fix arbitrary $t, t' \in \mathbb{R}$ and let
\[
\mathbf{z}(\tau) := \mathbf{z}^{(t)}(\tau), \qquad 
\tilde{\mathbf{z}}(\tau) := \mathbf{z}^{(t')}(\tau), \quad \tau \in [0, w].
\]

Then, for almost every $\tau \in [0, w]$,
\begin{align*}
\frac{d}{d\tau} \|\mathbf{z}(\tau) - \tilde{\mathbf{z}}(\tau)\|
&\leq \bigl\| \mathbf{f}(\mathbf{x}(t+\tau), \mathbf{z}(\tau)) 
- \mathbf{f}(\mathbf{x}(t'+\tau), \tilde{\mathbf{z}}(\tau)) \bigr\| \\
&\leq L_f \bigl( \|\mathbf{z}(\tau) - \tilde{\mathbf{z}}(\tau)\| 
+ \|\mathbf{x}(t+\tau) - \mathbf{x}(t'+\tau)\| \bigr).
\end{align*}

Since $\mathbf{x}$ is $L_x$-Lipschitz,
\[
\|\mathbf{x}(t+\tau) - \mathbf{x}(t'+\tau)\| \leq L_x |t - t'|.
\]

Substituting gives the differential inequality
\[
\frac{d}{d\tau} \|\mathbf{z}(\tau) - \tilde{\mathbf{z}}(\tau)\| 
\leq L_f \|\mathbf{z}(\tau) - \tilde{\mathbf{z}}(\tau)\| + L_f L_x |t - t'|.
\]

Applying Grönwall's inequality (integral form) with initial condition 
$\mathbf{z}(0) = \tilde{\mathbf{z}}(0) = \mathbf{0}$ yields
\begin{align*}
\|\mathbf{z}(w) - \tilde{\mathbf{z}}(w)\|
&\leq L_f L_x |t - t'| \int_0^w e^{L_f (w - \sigma)}\, d\sigma \\
&= L_f L_x |t - t'| \cdot \frac{e^{L_f w} - 1}{L_f} \\
&= L_x \frac{e^{L_f w} - 1}{L_f} |t - t'|.
\end{align*}

Since $\boldsymbol{\Phi}(t \cdot w) = \mathbf{z}(w)$ and 
$\boldsymbol{\Phi}(t' \cdot w) = \tilde{\mathbf{z}}(w)$, we obtain
\[
\|\boldsymbol{\Phi}(t \cdot w) - \boldsymbol{\Phi}(t' \cdot w)\| 
\leq L_x \frac{e^{L_f w} - 1}{L_f} |t - t'|.
\]

Dividing both sides by $w > 0$ gives the claimed Lipschitz constant 
for the rescaled time variable:
\[
\|\boldsymbol{\Phi}(s) - \boldsymbol{\Phi}(s')\| 
\leq \Bigl( L_x \frac{e^{L_f w} - 1}{L_f w} \Bigr) |s - s'|,
\]
where $s = t \cdot w$ and $s' = t' \cdot w$. This completes the proof.
\end{proof}

\subsection{Irregularity Lemma}
\label{ap:irreg_lemma}
\begin{lemma*}[\ref{lemma:irreg}]
    Consider an irregular sequence, with observations at times~$\mathcal{T} = t_1,t_2\ldots, t_n$, such that the inter-observation gaps are i.i.d. random variables: $\Delta t_i = t_{i+1} - t_i$, with a coefficient of variation (standard deviation over mean) of $\alpha \in \mathbb{R}$.
    Also consider another sequence $\mathcal{T}' = t_1', t_2', \ldots, t_m'$, which is created from $\mathcal{T}$ through our patched aggregation procedure, i.e. by taking every $s$-th time, where $s \in \mathbb{N}$.
    Then, the inter-observation gaps of~$\mathcal{T}'$, denoted by $\Delta t_i' = t_{i+1} - t_i$ will have a coefficient of variation $\alpha' = \frac{\alpha}{\sqrt{s}}$, i.e. $\sqrt{s}$ times smaller.
\end{lemma*}
\begin{proof}
    Since the step is $s$, each new inter-observation gap will be a sum of $s$ old inter-observation gaps:
    $$
    \Delta t_i' = \sum_{j \in \mathcal{S}_i} \Delta t_j;\; |\mathcal{S}_i| = s.
    $$
    Here, we denote all the old times between $t_{i+1}'$ and $t_i'$ by $\mathcal{S}_i$ for convenience.

    Next, let's assume that we know the moments of the old inter-observation gaps:
    $$
    \mu := \mathbb{E} \Delta t_i;
    \qquad
    \sigma^2 := \textrm{Var} (\Delta t_i).
    $$
    Then, the old variance coefficient is~$\frac{\sigma}{\mu}$.
    Now we may calculate the new variance coefficient, using the fact that the sum of i.i.d. random variables has mean equal to the sum of means, and variance equal to the sum of variances:
    $$
    \mu' := \mathbb{E} \Delta t_i' = s \mu; \qquad \sigma'^2 := \textrm{Var} (\Delta t_i') = s \sigma^2.
    $$
    Now we may directly calculate the new variation coefficient:
    $$
    \frac{\sigma'}{\mu'} = \frac{\sqrt{s} \sigma}{s \mu} = \frac{1}{\sqrt{s}} \frac{\sigma}{\mu},
    $$
    which concludes our proof
\end{proof}

We complement Lemma~\ref{lemma:irreg} with Figure~\ref{fig:cv_study}, which plots the coefficient of variation of time intervals against level number on our irregular benchmark, ETTm1 with 30\% of days discarded.
The result precisely aligns with our reasoning, with the coefficient of variation decreasing exponentially as we go deeper.

\begin{figure}
    \centering
    \includegraphics[width=0.8\linewidth]{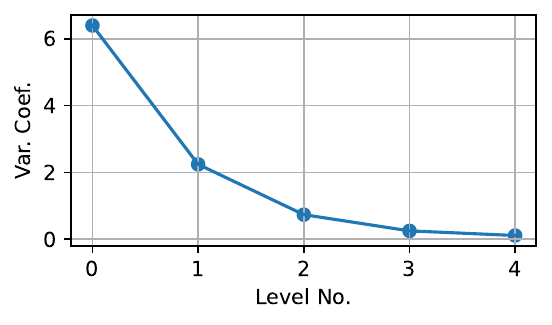}
    \caption{Coefficient of Variation of inter-observation gaps vs Number of Patching Levels, on the irregular ETTm1.}
    \label{fig:cv_study}
    \Description{The plot shows an exponentially decaying Coefficient of Variation.}
\end{figure}

\section{Additional Results}
\label{ap:results}

\subsection{Attention Visualization}
To explore the learned attention patterns, we visualize several patches, which correspond to maximum attention weights.

\begin{figure}[tbh]
    \centering
    \includegraphics[width=0.48\linewidth]{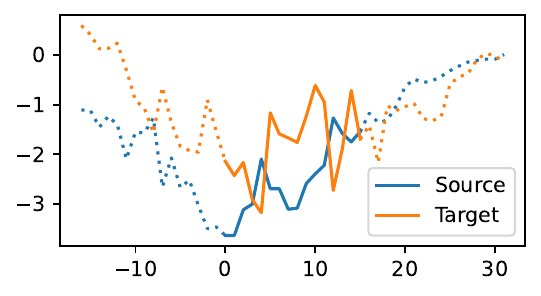}
    \includegraphics[width=0.48\linewidth]{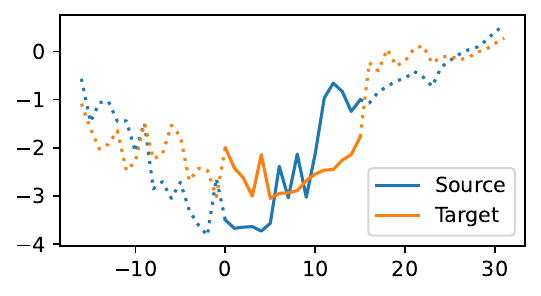}
    \caption{Two examples of visually similar patches, corresponding to locations with top attention weights for the encoder self-attention.}
    \label{fig:encoder_attention}
    \Description{The plot shows visually similar examples of large-self-attention-weights patches.
    The similarity is evident, but limited.
    }
\end{figure}
\begin{figure}[tbh]
    \centering
    \includegraphics[width=0.48\linewidth]{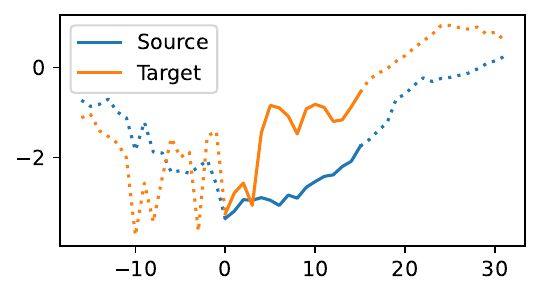}
    \includegraphics[width=0.48\linewidth]{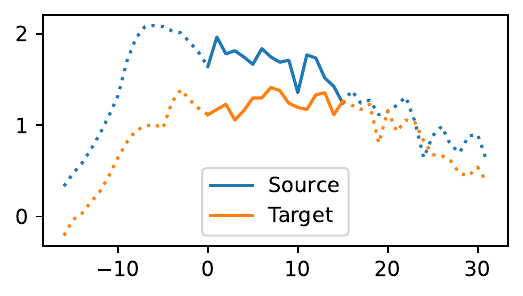}
    \caption{}
    \caption{Two examples of visually similar patches, corresponding to locations with top attention weights for the decoder cross-attention.}
    \label{fig:decoder_attention}
    \Description{The plot shows visually similar examples of large-cross-attention-weights patches, more similar than those in the self-attention version.
    }
\end{figure}

For testing the \textbf{encoder self-attention} patterns, we visualize the corresponding query and key patches from $x$, aligning them for easier cognition.
The results are presented in Figure~\ref{fig:encoder_attention}.
We find that the patches, although often dissimilar themselves, are often located in similar neighbourhoods.
We also note, that encoder patch similarity is not critical, as the encoder self-attention is directed at feature extraction and propagating information throughout the sequence.

To investigate the \textbf{decoder cross-attention} patterns, we fix the generation seed to $0$ (since attention maps may vary sample-to-sample), and visualize the patches from $x$ and $y$, corresponding to maximum attention weights.
The results are in Figure~\ref{fig:decoder_attention}.
Note, that the model does not see the target patch, and is forced to guess, what kind of patch that will be.
Nevertheless, the similarities of top patches found by the model are often striking, indicating that the model can predict the general patch type rather well.

\subsection{Short term forecasting}
\label{ap:short-term}
Although our method, by design, is better suited for longer horizons, we also test it on the short-term forecasting task, comparing its performance to one of the runner-ups, $K^2$VAE, on the regular ETTm datasets.
Both the horizon and the history lengths are 24 here, so we make the UFO model smaller so that it can process such short sequences.

\begin{table}[!tb]
    \centering
    \caption{Results on the short-term forecasting task: length of history and length of horizon are both 24, on the ETT regular dataset.}
    \label{tab:short_results}
    \begin{tabular}{lrrr}\toprule
    Method &ETTm1 &ETTm2 \\\midrule
    $K^2$VAE &0.267 +- 0.002 &0.363 +- 0.015 \\
    UFO-tiny &0.246 +- 0.006 &0.234 +- 0.002 \\
    \bottomrule
    \end{tabular}
\end{table}

The results are presented in Table~\ref{tab:short_results}.
As you can see, not only does our method still outperform $K^2VAE$, the gap seems to get bigger.

\section{Experiment Details}
\label{ap:exp_details}
Most of our benchmark follows the setup in $K^2$VAE~\cite{wu2025k}, and the preceding ProbTS benchmark~\cite{zhang2024probts}.
Below, we explain the datasets used, and the methods we compare to in greater detail.

\subsection{Data description}
As mentioned in the main text, we use 5 popular multivariate time series datasets: ETTm1, ETTm2, Weather, Electricity, Traffic.
The respective statistics are presented in Table~\ref{tab:datastats}.

\paragraph{ETT (Electricity Transformer Temperature)}
The ETT datasets consist of multivariate time series collected from electricity transformers in two different counties in China, denoted as ETTm1/ETTh1 and ETTm2/ETTh2 (or collectively ETT1 and ETT2). They record 7 variables, including transformer oil temperature (the main target) and six load-related features. Measurements are taken at 15-minute intervals in the minute-level versions (ETTm1, ETTm2) over roughly two years, yielding 69{,}680 timesteps, and at hourly intervals in the hour-level versions (ETTh1, ETTh2), yielding 17{,}420 timesteps. ETT1 (typically ETTh1/ETTm1) exhibits relatively cleaner patterns and is considered easier, with many prior methods achieving strong performance. ETT2 (ETTh2/ETTm2) is noticeably noisier and more challenging, often revealing limitations in existing forecasting approaches.
We opt for the "m" versions of these datasets to maximize the amount of data.
To maximize dataset diversity, while staying within our compute constraints, we skip the "h" versions, opting for other datasets instead (described below).

\paragraph{Weather}
The Weather dataset comprises multivariate local meteorological observations, including temperature, humidity, wind speed, atmospheric pressure, and other climatological variables (21 variables in total). Recorded at 10-minute intervals, it spans over 52{,}000 timesteps (approximately 52{,}696 in the standard benchmark version). Thanks to its pronounced daily and yearly seasonal cycles, clear trends, and gradual shifts, it serves as a widely adopted benchmark for evaluating both short-term and long-term forecasting models, particularly those designed to capture periodic meteorological dynamics.

\paragraph{Electricity}
The Electricity dataset records hourly electricity consumption (in kW) across 320 clients over several years, resulting in 26{,}304 timesteps in the standard benchmark version. It exhibits strong daily, weekly, and seasonal periodicity driven by human activity, time-of-day patterns, weekdays/weekends, and broader demand fluctuations. As one of the most established multivariate forecasting benchmarks, it is routinely used in short-term and long-term settings to test a model's ability to handle cross-consumer patterns, recurring cycles, and external influencing factors.

\paragraph{Traffic}
The Traffic dataset contains hourly road occupancy rates (normalized to [0, 1]) measured by 862 sensors across freeways in the San Francisco Bay Area, spanning 17{,}544 timesteps in the standard version. The data display prominent hourly and daily periodic patterns with relatively stable long-term levels and minimal overall trends. Widely employed in intelligent transportation system research, it provides a challenging yet clean multivariate benchmark for assessing how effectively models capture strong seasonality and recurring traffic flow behaviours in real-world urban mobility scenarios.

\begin{table}[tbh]
    \caption{Dataset statistics.}
    \label{tab:datastats}
    \begin{tabular}{lrrrrr}\toprule
    Name & Variates & Total Days & Observations & Frequency \\\midrule
    ETTm1 &7 &725 &69680 &15m \\
    ETTm2 &7 &725 &69680 &15m \\
    Weather &21 &365 &52695 &10m \\
    Electricity &320 &1095 &26304 &1h \\
    Traffic &862 &730 &17544 &1h \\
    \bottomrule
    \end{tabular}
\end{table}

\subsection{Additional UFO details}
\paragraph{Loss function}
As we mentioned in Section~\ref{sec:setup}, we use NCRPS as the loss function for UFO.
It is differentiable, so we can pass gradients through.
Additionally, as we mentioned above, CRPS is strictly proper, so it may be used for training without fear that the model may end up overfitting to some local optima of this metric.
We sample $16$ randomized predictions per example in-parallel, calculate NCRPS, and backpropagate through the model, using the reparametrization trick to pass gradients through the top level of UFO.

\paragraph{Normalization}
In our experiments, we found that the normalization strategy plays a critical role in training stability.
The optimal configuration used in all our presented UFO results is as follows:
\begin{itemize}
    \item \textbf{Input/Output:} Reversible Instance Normalization~\cite{kim2021reversible}.
    \item \textbf{Pre-up/downsampling:} Instance Normalization.
    \item \textbf{Transformers:} Instance Normalization.
    \item \textbf{Top layer:} Layer Normalization.
\end{itemize}

\subsection{CRPS Calculation} \label{ap:crps}
\paragraph{Per-channel procedure}
For a univariate target $y$ and predictive CDF $F$,
\begin{equation}
\mathrm{CRPS}(F, y) =
\int_{-\infty}^{\infty}
\bigl(F(z) - \mathbbm{1}\{y \le z\}\bigr)^2 \, dz.
\label{eq:crps_single}
\end{equation}
Calculating the integral from~\eqref{eq:crps_single} using sample predictions is non-trivial.
Instead, we use an equivalent CRPS expression~\cite{gneiting2005calibrated}, which is simple to evaluate over samples:
$$
\textrm{CRPS}(F, x) = \mathbb{E}_{X \sim F}\left[ \left| X - x \right|\right] - \frac12 \mathbb{E}_{X,X^* \sim F} \left[ \left| X - X^* \right| \right].
$$
The first term is simply the mean absolute error.
The second term can be efficiently calculated in $\mathcal{O}(P \log P)$ (where $P$ is the number of samples), by sorting the samples and calculating the coefficient for each sample.
For a sorted sequence, starting from the smallest one, the coefficients are: $-P+1, -P+3, \dots, P-3, P-1 := [-P+1:P-1:2]$ (where we use the brackets notation to denote [start:stop:step], with start and stop inclusive).
Consequently, the formula for CRPS from samples looks like:
$$
\textrm{CRPS}(X, x) = \textrm{MAE}(X, x) + \sum_{k \in [-P+1:P-1:2]} k X[k].
$$
Here, $X[k]$ denotes the $k-th$ smallest sample.

\paragraph{Averaging, Normalization}
While CRPS can be averaged over time, multivariate series often have channels with heterogeneous scales, making naive averaging (as done in~\cite{wu2025k}) across dimensions undesirable. 
We therefore normalize each channel independently and define the \textbf{Normalized CRPS} (NCRPS)~as
\begin{equation}\label{eq:ncrps}
\mathrm{NCRPS}(\mathbf{F}, \mathbf{y}) =
\frac{1}{d} \sum_{j=1}^d
\frac{\sum_{i=1}^L \mathrm{CRPS}(F_{i,j}, y_{i,j})}
     {\sum_{i=1}^L |y_{i,j}|}.
\end{equation}
The first index $i$ denotes the observation number, the second $j$ is responsible for channel number.
The $\ell_1$ normalization yields an interpretable, relative error–like metric across channels, which we seek to minimise. 

\subsection{Additional Baseline Details}

\paragraph{Non-ODE Baselines}
Most non-ODE baselines were taken from the $K^2$VAE repository\footnote{https://github.com/decisionintelligence/K2VAE}, along with the corresponding (tuned) hyperparameters, otherwise we selected reasonable hyperparameters empirically.

\paragraph{ODE Baselines}
For all ODE baselines, we chose model-specific hyperparameters according to the original repositories.
Other more general hyperparameters (e.g. hidden size) were chosen empirically, in accordance with other models included in the comparison.
In all ODE baselines (where applicable) we set tolerance to $0.001$ to speed up execution.
For Latent ODE (reconstruction term), NCDE, NCDE++ and UFO we used the NCRPS metric as the loss function, sampling 16 times per training example.
For SLiCE, since it acts autoregressively, we chose the DeepAR-like autoregressive sampling procedure, along with the corresponding log-likelihood loss function.
NCDE and NCDE++ were implemented in an encoder-decoder like design, similar to UFO, where the input embedding is used to produce a diagonal Gaussian distribution, which is then used to sample the initial state for the decoder.

\end{document}